\documentclass[journal]{IEEEtran}
\usepackage{graphicx}
\usepackage{amsmath,amsfonts}
\usepackage{amssymb}
\usepackage{amsthm}
\usepackage{booktabs}
\usepackage{multirow}
\usepackage{algorithmic}
\usepackage{algorithm}
\usepackage{array}
\usepackage[caption=false,font=footnotesize]{subfig}
\usepackage{textcomp}
\usepackage{stfloats}
\usepackage{url}
\usepackage{verbatim}
\usepackage{cite}
\usepackage{color}
\usepackage{xspace}
\usepackage{mathtools}
\usepackage{orcidlink}
\newtheorem{assumption}{Assumption}
\newtheorem{property}{Property}
\newtheorem{proposition}{Proposition}

\DeclareMathOperator{\EX}{\mathbb{E}}
\newcommand{\etal}{\textit{et al.}}
\newcommand{\ourmethod}{\textit{CARPM-FIQA}\xspace}

\usepackage[T1]{fontenc}
\usepackage{newtxtext}
\usepackage{newtxmath}
\usepackage{placeins}

\begin{document}

\title{Learning Steadily: Accumulating Relative Point Margin Scores for Face Image Quality Assessment}

\author{
Guray Ozgur\,\orcidlink{0000-0002-1966-6641},
Tahar Chettaoui\,\orcidlink{0009-0004-9744-7235},
Eduarda Caldeira\,\orcidlink{0009-0002-4891-0057},
Marco Huber\,\orcidlink{0000-0003-3413-6291},
Jan Niklas Kolf\,\orcidlink{0000-0002-0037-5334},\\
Naser Damer\,\orcidlink{0000-0001-7910-7895}, \IEEEmembership{Senior Member, IEEE},
and Fadi Boutros\,\orcidlink{0000-0003-4516-9128}

\thanks{This research work has been funded by the German Federal Ministry of Education and Research and the Hessen State Ministry for Higher Education, Research and the Arts within their joint support of the National Research Center for Applied Cybersecurity ATHENE.}
\thanks{Guray Ozgur, Tahar Chettaoui, Eduarda Caldeira, Marco Huber, Jan Niklas Kolf, Naser Damer, and Fadi Boutros are with Fraunhofer Institute for Computer Graphics Research IGD, Fraunhoferstr. 5, 64283 Darmstadt, Germany.}
\thanks{Guray Ozgur, Tahar Chettaoui, Eduarda Caldeira, Marco Huber, Jan Niklas Kolf, and Naser Damer are also with the Department of Computer Science at the TU Darmstadt, Karolinenpl. 5, 64289 Darmstadt, Germany.}
\thanks{Accepted for publication in IEEE Transactions on Biometrics,
Behavior, and Identity Science. Digital Object Identifier (DOI):
10.1109/TBIOM.2026.3734769.}

\thanks{\copyright\ 2026 IEEE. Personal use of this material is
permitted. Permission from IEEE must be obtained for all other uses,
in any current or future media, including reprinting/republishing this
material for advertising or promotional purposes, creating new
collective works, for resale or redistribution to servers or lists,
or reuse of any copyrighted component of this work in other works.}
}

\maketitle
\begin{abstract}
Face Image Quality Assessment (FIQA) determines the suitability of captured face images for automated face recognition (FR), a critical capability for reliable biometric systems. Existing state-of-the-art (SOTA) FR-integrated FIQA methods suffer from temporal instability: as the feature space evolves during training, single-epoch quality estimates fluctuate, creating a moving target that undermines reliable quality prediction. We introduce \ourmethod, a stabilization strategy for FR-integrated FIQA that accumulates relative point margin measurements, the ratio between intra-class compactness and inter-class separation, across the entire training trajectory rather than relying on single-epoch estimates. This cumulative averaging approach provides theoretically grounded advantages: reduced variance in quality estimates, improved mean squared error, and enhanced ranking stability with convergence guarantees as training progresses. Through controlled experiments on the SynFIQA dataset with labeled quality groups, we demonstrate that cumulative averaging achieves superior discriminative ability, and ablation studies across different training configurations confirm consistent improvements. Evaluated against twelve FIQA methods on eight challenging benchmarks with four FR models at two FMR thresholds, \ourmethod places 4th (\ourmethod(L)) and 6th (\ourmethod(S)) of 17 compared methods by pAUC-EDC and AUC-EDC averaged across FR models and, after per-benchmark normalization, across benchmarks, staying within a few percent of the best method's normalized average for every FR model, providing a principled solution to training instability while maintaining the performance benefits of FR integration. More broadly, our work demonstrates that temporal aggregation strategies can stabilize training objectives in deep learning systems where target values inherently fluctuate due to evolving feature representations. \\ Code: \url{https://github.com/gurayozgur/CARPM-FIQA}.
\end{abstract}


\begin{figure}[t]
    \centering
    \includegraphics[width=\linewidth]{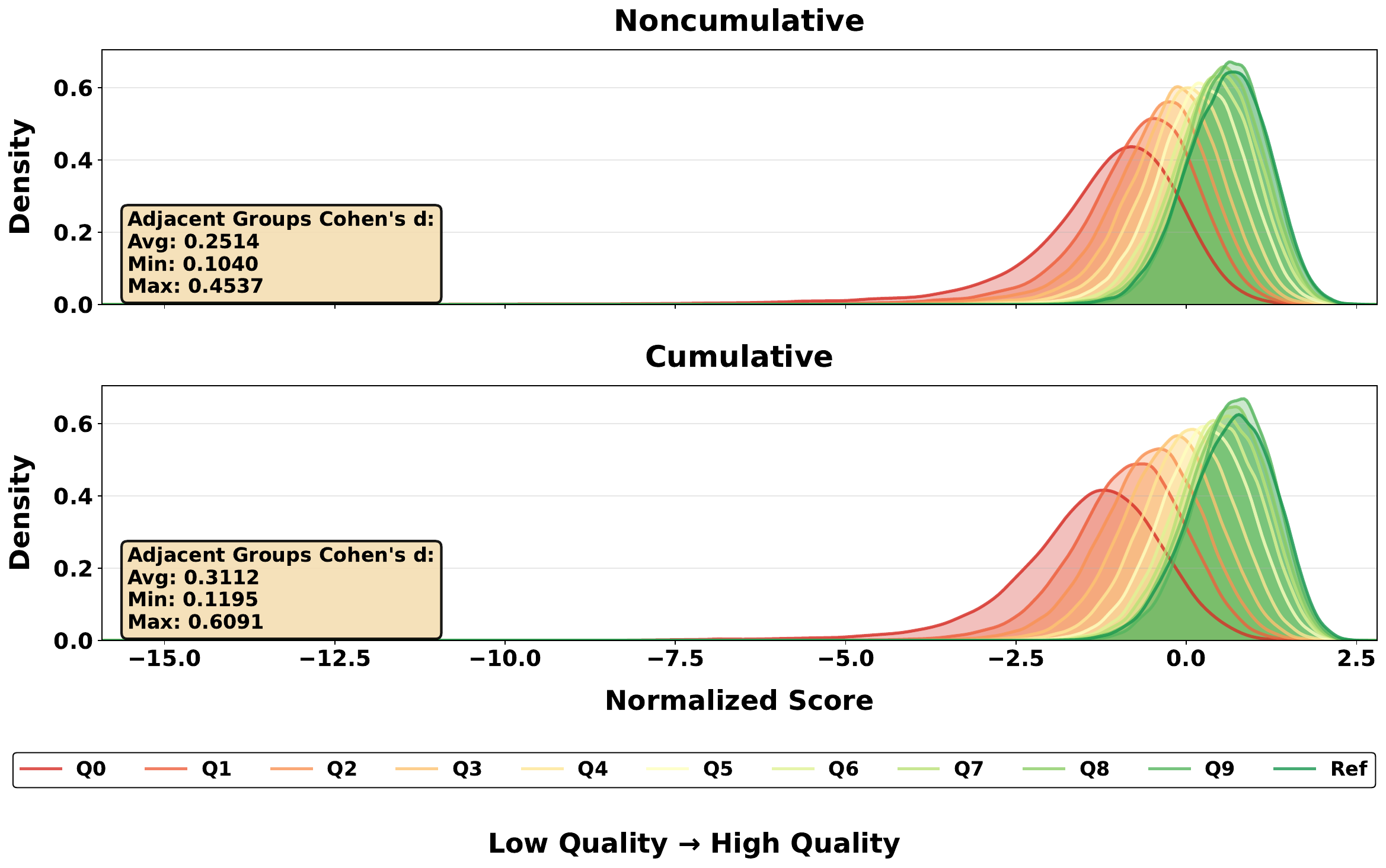}
    \caption{Normalized quality score distributions for 11 quality groups, each having 0.5M images, from 5.5M images of SynFIQA \cite{mrfiqa}. (Q0-Q9: degraded images from lowest to highest quality; Ref: reference images) comparing non-cumulative (top) and cumulative (bottom) methods. Scores are z-score normalized within each method. Color gradient from red to green indicates quality progression. Cohen's d effect sizes between adjacent groups are reported. The cumulative method achieves higher average Cohen's d (0.3112 vs.\ 0.2514) on the small protocol (ArcFace/ResNet50/CASIA-WebFace), demonstrating a better group separation.
    }
    \label{fig:normalized_distributions_casia}
    \vspace{-6mm}
\end{figure}

\newpage
\section{Introduction}
\label{sec:intro}

\IEEEPARstart{F}{IQA} evaluates face image utility for face recognition processing, measuring \textit{recognition utility} or \textit{suitability for identity verification} \cite{Quality_ISO,NISTQuaity}. Unlike general Image Quality Assessment (IQA) methods that assess quality from human perception \cite{BRISQE_IQA,nique,liu2017rankiqa}, FIQA specifically quantifies how effectively a facial image serves automated recognition tasks such as deciding whether images from passport scans or live captures contain sufficient biometric information for accurate identity matching \cite{NISTQuaity,DBLP:journals/csur/SchlettRHGFB22}. As demonstrated in \cite{BiyingWACV}, high perceived quality does not always correlate with FR utility, particularly when factors like facial occlusions are present. This distinction explains why FIQA approaches consistently outperform general IQA methods for FR applications \cite{boutros_2023_crfiqa,MagFace,SDDFIQA,DBLP:conf/iwbf/BabnikDS23}. Current SOTA FR-integrated FIQA approaches \cite{boutros_2023_crfiqa,MagFace} tightly couple quality assessment with FR training, leading to superior performance on challenging benchmarks \cite{yang2025fate}. However, these methods can suffer from unstable quality estimates due to the dynamic nature of the training process. As the feature space evolves during training, the induced clustering structure changes, leading to fluctuations in quality estimates across training iterations. This instability makes it difficult to consistently identify high-quality face images that remain distinguishable throughout training, creating a moving target for quality assessment. Consider a sample that appears high-quality at epoch $t$ based on its position in the embedding space, this same sample may be classified as medium-quality at epoch $t+k$ as class boundaries shift and competing identities reorganize. Such temporal inconsistency undermines the reliability of quality predictions, particularly for samples near decision boundaries.

To address this fundamental challenge, we propose \ourmethod (\textit{\textbf{C}umulative \textbf{A}verage of \textbf{R}elative \textbf{P}oint \textbf{M}argin for \textbf{FIQA}}), which stabilizes quality estimates by accumulating measurements across the training trajectory. \ourmethod is a temporal stabilization strategy for FR-integrated FIQA: it deliberately adopts the established relative point margin/classifiability score of prior work (Sec. \ref{subsec:related_metrics}) as its per-epoch signal, and contributes the principled temporal stabilization this family of methods has lacked, transforming an inherently fluctuating single-epoch estimate into a provably stable training target with formal variance-reduction, MSE, and ranking-convergence guarantees. Our approach makes three key contributions: (1) A quality stabilization approach that accumulates relative point margin values, measuring the ratio between intra-class compactness and inter-class separation, across training epochs, providing temporally-stable quality estimates that capture long-term discriminative properties of face images. (2) A theoretical framework demonstrating that cumulative averaging reduces variance, improves mean squared error, and enhances ranking stability, with formal proofs and empirical validation. (3) A regression model that effectively predicts accumulated quality scores for unseen face images, generalizing the temporal stability patterns observed during training to new samples without requiring access to their training history.

Through extensive evaluation on eight challenging benchmarks (LFW \cite{LFWTech}, AgeDB-30 \cite{agedb}, CFP-FP \cite{cfp-fp}, CALFW \cite{CALFW}, Adience \cite{Adience}, CPLFW \cite{CPLFWTech}, XQLFW \cite{XQLFW}, and IJB-C \cite{ijbc}) across six SOTA FR models (ArcFace \cite{deng2019arcface}, ElasticFace \cite{elasticface}, MagFace \cite{MagFace}, CurricularFace \cite{curricularFace}, TransFace \cite{transface}, SwinFace \cite{swinface}), we demonstrate that \ourmethod consistently ranks among the top-performing methods, effectively addressing the instability issues present in existing FR-integrated FIQA approaches while maintaining the advantage of tight coupling with the FR task. \vspace{-2mm}

\begin{figure}[t]
    \centering
    \setlength{\tabcolsep}{1pt}
    \renewcommand{\arraystretch}{1.1}
    \includegraphics[width=0.9\linewidth]{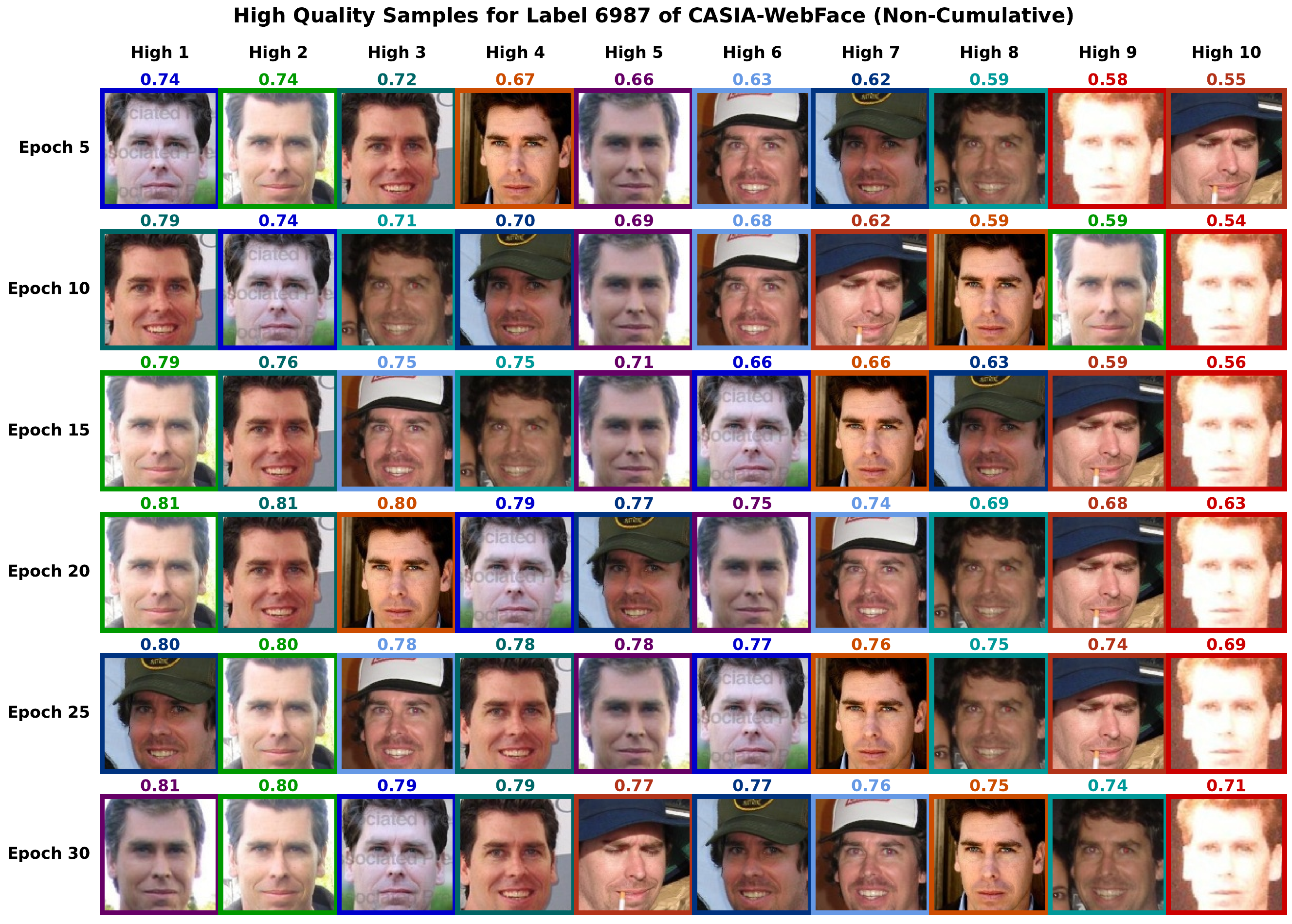}\\
    \includegraphics[width=0.9\linewidth]{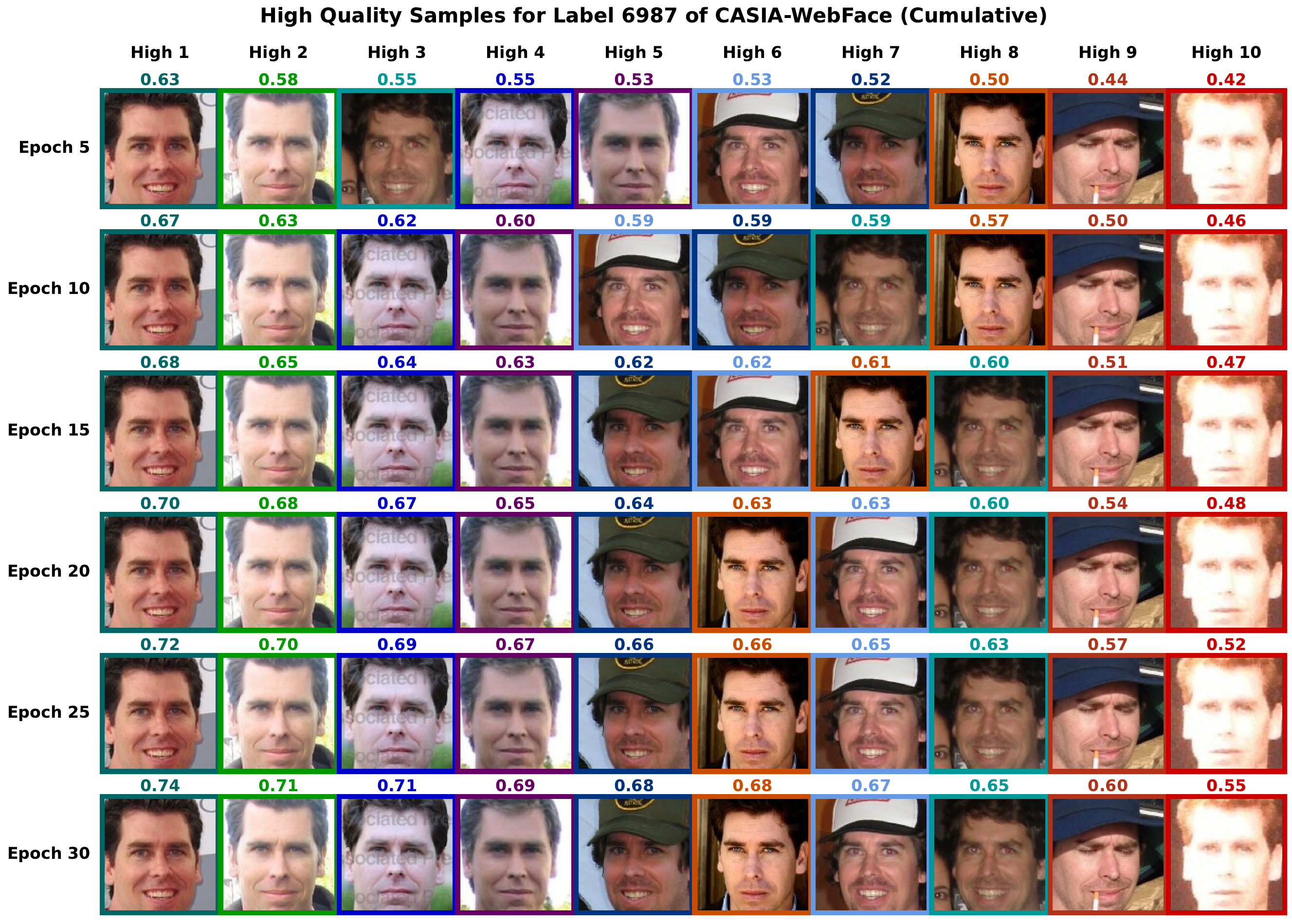}
    \caption{High quality samples from identity \#6987 across training epochs. Note how the top-ranked images in the non-cumulative approach (top) tend to stabilize in the cumulative approach (bottom), showing the benefit of averaging quality scores over time.}
    \label{fig:high_quality_id6987}
    \vspace{-5mm}
\end{figure}

\section{Related Work}
\label{sec:related_work}

\subsection{Face Recognition Models and Training}
\label{subsec:related_fr}

Deep learning has revolutionized FR through CNNs \cite{deng2019arcface,DBLP:conf/cvpr/WangWZJGZL018,DBLP:conf/cvpr/LiuWYLRS17} and more recently through Vision Transformers (ViTs) \cite{transface,DBLP:conf/cvpr/KimS0JL24,DBLP:journals/ivc/ChettaouiDB25}. The key to FR success lies in loss functions that explicitly encourage intra-class compactness and inter-class separation. The Angular Margin Penalty-based Softmax Loss family \cite{deng2019arcface,curricularFace,MagFace,elasticface,DBLP:conf/cvpr/WangWZJGZL018} achieves this by introducing angular margin penalties between deep features and their corresponding class centers. ArcFace \cite{deng2019arcface} adds an angular margin in the arc-cosine space, CosFace \cite{DBLP:conf/cvpr/WangWZJGZL018} applies a cosine margin, while CurricularFace \cite{curricularFace} adaptively emphasizes hard samples. These margin-based losses induce a feature space where samples cluster tightly around their identity centers while maintaining large margins from competing identities, a structure that directly influences the stability and reliability of quality estimates derived during training. \vspace{-2mm}

\subsection{Quality Metrics and Class Separability}
\label{subsec:related_metrics}

The silhouette score \cite{ROUSSEEUW198753silhouette} provides a foundational framework for evaluating clustering quality \cite{lovmar2005silhouette, januzaj2023silhouette, ogbuabor2018silhouette, shahapure2020silhouette}. It quantifies how well a data point is assigned to its cluster by considering both intra-cluster cohesion and inter-cluster separation. For classification tasks, we can adapt this concept by using class instead of cluster. Formally, the silhouette score for a sample \( i \) is defined as: \vspace{-2mm}
\begin{equation}
    s(i) = \frac{b(i) - a(i)}{\max(a(i), b(i))}
    \label{eq:silhouette} ,
\end{equation}
where \( a(i) \) represents the average distance between sample \( i \) and all other points within the same class (intra-class compactness), and \( b(i) \) denotes the smallest average distance between \( i \) and all points in the nearest different class (inter-class separation). Values close to \( +1 \) indicate well-separated samples, values around \( 0 \) suggest samples near decision boundaries, and negative values imply potential misclassification. In FR models trained with Angular Margin Penalty-based Softmax Loss \cite{deng2019arcface, curricularFace, MagFace, elasticface, DBLP:conf/cvpr/WangWZJGZL018, DBLP:conf/cvpr/LiuWYLRS17, MagFace}, class centers \( w_{y_i} \) are stored in the classification layer's weight matrix. For efficiency, \( a(i) \) and \( b(i) \) can be redefined using cosine similarity with class centers rather than computing distances to all class members, as done by CR-FIQA \cite{boutros_2023_crfiqa}: \( a(i) := \cos(\theta_{y_i}) \) where \( \cos(\theta_{y_i}) = x_i \cdot W_{y_i} \), and \( b(i) := \max_{y \neq y_i} \cos(\theta_y) \). This reformulation shows that the silhouette score's information derives from the ratio of intra-class to inter-class distances. This aligns with the \textit{Relative Point Margin} (RPM) concept introduced by Ackerman \etal \cite{DBLP:conf/nips/Ben-DavidA08}, defined as \( RPM(x) = d(x, c_x)/d(x, c_{x_0}) \) where \( c_x \) is the closest center and \( c_{x_0} \) is the second closest. They show that relative margin measures derived from RPM satisfy Kleinberg's axioms for clustering quality \cite{DBLP:conf/nips/Kleinberg02, DBLP:conf/nips/Ben-DavidA08}. This theoretical foundation has been leveraged for FIQA applications, as we discuss next. \vspace{-2mm}

\subsection{Face Image Quality Assessment}
\label{subsec:related_fiqa}

FIQA methods can be categorized into four distinct groups, each with unique characteristics and trade-offs. \textbf{(1) Label-generation approaches} train regression networks using quality labels from various sources. FaceQnet \cite{faceqnetv1} uses ICAO compliance standards as quality references, while SDD-FIQA \cite{SDDFIQA} employs distribution distances between embeddings. RankIQ \cite{RANKIQ_FIQA} adopts a learning-to-rank strategy, training models to predict quality rankings based on FR performance metrics across different datasets, better capturing relative differences between samples. A notable limitation is that these approaches often decouple FIQA from FR, typically employing shallower networks that don't fully leverage the deep facial features captured by SOTA FR models. \textbf{(2) Non-FR model approaches} include DifFIQA \cite{10449044}, which leverages diffusion models to assess embedding robustness by exploring the stability of face representations under different noise conditions, and eDifFIQA \cite{babnikTBIOM2024}, which distills this approach into a lighter model through knowledge distillation for faster inference. While these methods can achieve high accuracy, they incur significant computational costs, especially when employing large generative models. \textbf{(3) Pre-trained FR analysis approaches} operate on fixed FR models without requiring additional training. SER-FIQ \cite{SERFIQ} measures embedding stability under dropout perturbations by evaluating consistency with varied dropout patterns across multiple forward passes. GraFIQs \cite{grafiqs} uses gradient magnitudes during backpropagation to evaluate how strongly each sample aligns with the FR model's optimization objective. FaceQAN \cite{FaceQAN} estimates quality by quantifying adversarial robustness. These methods leverage existing FR models but are constrained by their fixed representations and often require multiple inference passes or backpropagation. \textbf{(4) FR-integrated approaches} directly incorporate quality assessment into the FR training process, achieving the tightest coupling between quality estimation and recognition objectives. MagFace \cite{MagFace} links quality scores to embedding magnitudes through adaptive margin penalties and magnitude-aware losses. PFE \cite{PFE_FIQA} models face embeddings as Gaussian distributions where the uncertainty represents quality. CR-FIQA \cite{boutros_2023_crfiqa} employs the silhouette-based framework to estimate quality by predicting a sample's relative classifiability. Specifically, CR-FIQA defines Closest Class Similarity (CCS) and Nearest Non-Class Center Similarity (NNCCS) as \( \text{CCS}(x) = \cos(x, W_{y_x}) \) and \( \text{NNCCS}(x) = \max_{y \neq y_x} \cos(x, W_y) \), where \( W_{y_x} \) is the class center for the ground truth identity and \( W_y \) represents competing class centers. The quality is then estimated as \( q(x) = \text{CCS}(x)/\text{NNCCS}(x) \), capturing how well-separated a sample is from competing identities. By optimizing against this target, CR-FIQA \cite{boutros_2023_crfiqa} aligns quality estimation with the induced clustering structure of the learned feature space. These FR-integrated approaches have consistently achieved top rankings in SOTA evaluations \cite{MagFace,boutros_2023_crfiqa}, demonstrating the advantages of fully incorporating quality assessment within the FR architecture. Despite the advances, existing FR-integrated methods rely on single-epoch measurements of quality-relevant properties (e.g., embedding magnitude at the final epoch, uncertainty at convergence, or classifiability at a specific training iteration). During training, each epoch generates a different induced class clustering, making quality estimation a moving target as \( q_j(x) \) varies across epochs \( j \). This causes not only fluctuating target values but also unstable sample rankings \cite{DBLP:conf/nips/GrillSATRBDPGAP20, DBLP:conf/icml/SutskeverMDH13}. These methods also typically select optimal checkpoints based on small benchmark performance and report results on larger datasets, introducing selection bias. Our method addresses these limitations by accumulating quality measurements across the entire training trajectory, providing a temporally-stable alternative that maintains the advantages of FR integration while improving reliability through cumulative averaging.\vspace{-1mm}

\begin{algorithm}[t]
\caption{CARPM-FIQA Training Procedure}
\label{alg:carpm_fiqa}
\begin{footnotesize}
\begin{algorithmic}[1]
\STATE \textbf{Input:} Training dataset $\mathcal{D} = \{(x_i, y_i)\}_{i=1}^{N}$, number of epochs $T$, hyperparameter $\lambda$
\STATE \textbf{Output:} Trained FR model with quality regression head
\STATE Initialize FR backbone network and quality regression head
\STATE Initialize cumulative quality scores: $Q(x_i)_0 = 0$ for all $i$
\FOR{epoch $n = 1$ to $T$}
    \FOR{each mini-batch $\mathcal{B} \subset \mathcal{D}$}
        \STATE // Forward pass: Extract embeddings
        \STATE Compute face embeddings $f_i$ for each sample $x_i \in \mathcal{B}$
        \STATE Compute predicted quality scores $\hat{Q}(x_i)_n$ from regression head
        \STATE // Compute current relative point margin (Eq. \ref{eq:silhouette})
        \STATE Compute $\cos(\theta_{y_i}) = f_i \cdot W_{y_i}$ (intra-class similarity)
        \STATE Compute $\max_{y \neq y_i} \cos(\theta_y)$ (inter-class similarity)
        \STATE Compute $q_n(x_i) = \frac{\cos(\theta_{y_i})}{\max_{y \neq y_i} \cos(\theta_y)}$
        \STATE // Update cumulative average quality score (Eq. \ref{eq:carpm})
        \STATE $Q(x_i)_n = \frac{q_n(x_i) + (n-1) \cdot Q(x_i)_{n-1}}{n}$
        \STATE // Compute losses
        \STATE Compute ArcFace loss: $\mathcal{L}_{Arc}$ \cite{deng2019arcface}
        \STATE Compute quality loss: $\mathcal{L}_{Q} = \frac{1}{|\mathcal{B}|}\sum_{x_i \in \mathcal{B}} \ell(Q(x_i)_n, \hat{Q}(x_i)_n)$ (Eq. \ref{eq:quality_loss})
        \STATE // Combined loss optimization (Eq. \ref{eq:total_loss})
        \STATE $\mathcal{L} = \mathcal{L}_{Arc} + \lambda \mathcal{L}_{Q}$
        \STATE Update model parameters via backpropagation
    \ENDFOR
\ENDFOR
\STATE \textbf{return} Trained FR model and quality regression head
\end{algorithmic}
\end{footnotesize} 
\end{algorithm}

\begin{figure}[t]
  \centering
  \subfloat[%
    EDC curves across epochs for non-cumulative (red) and cumulative (blue) targets. Better curves for cumulative (blue) targets.%
    \label{fig:edc_curves}
  ]{%
    \includegraphics[width=0.46\linewidth]{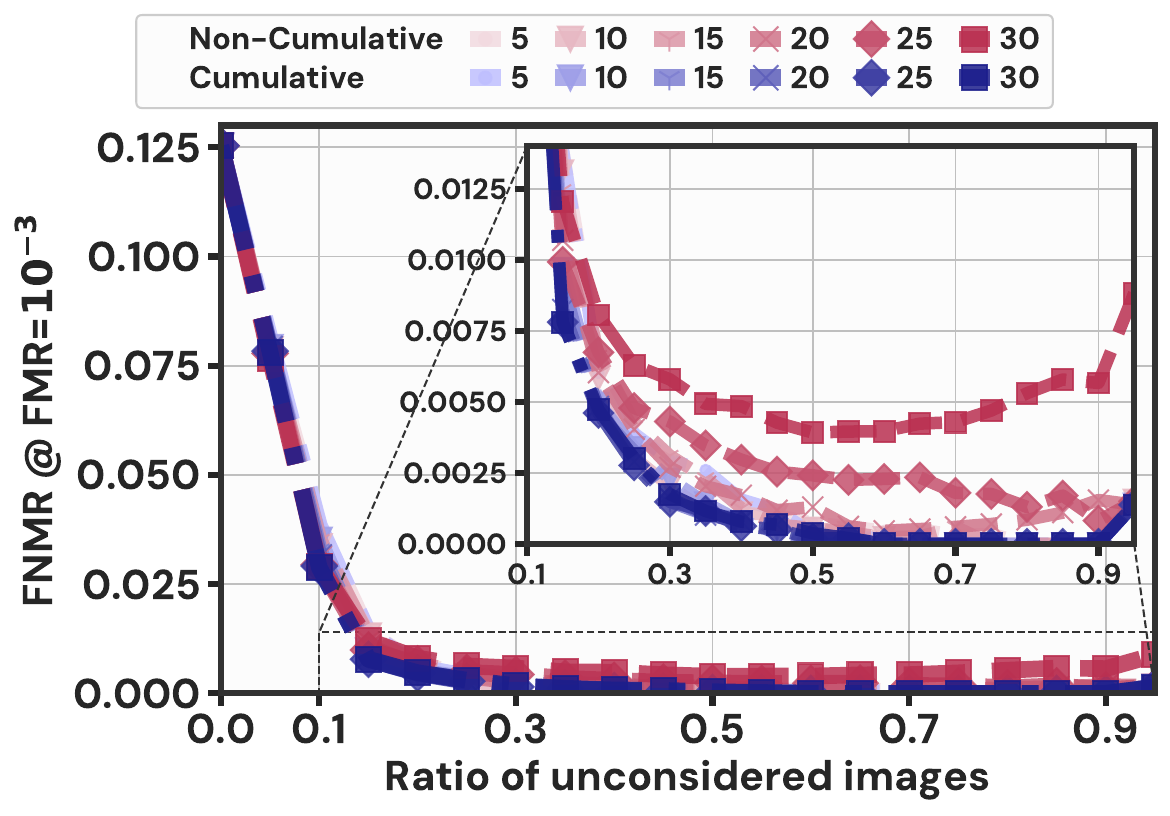}
  }
  \hfill
  \subfloat[%
    AUC-EDC trends. Lower values indicate better performance. Non-cumulative degrades after a sweet spot.%
    \label{fig:auc_comparison}
  ]{%
    \includegraphics[width=0.43\linewidth]{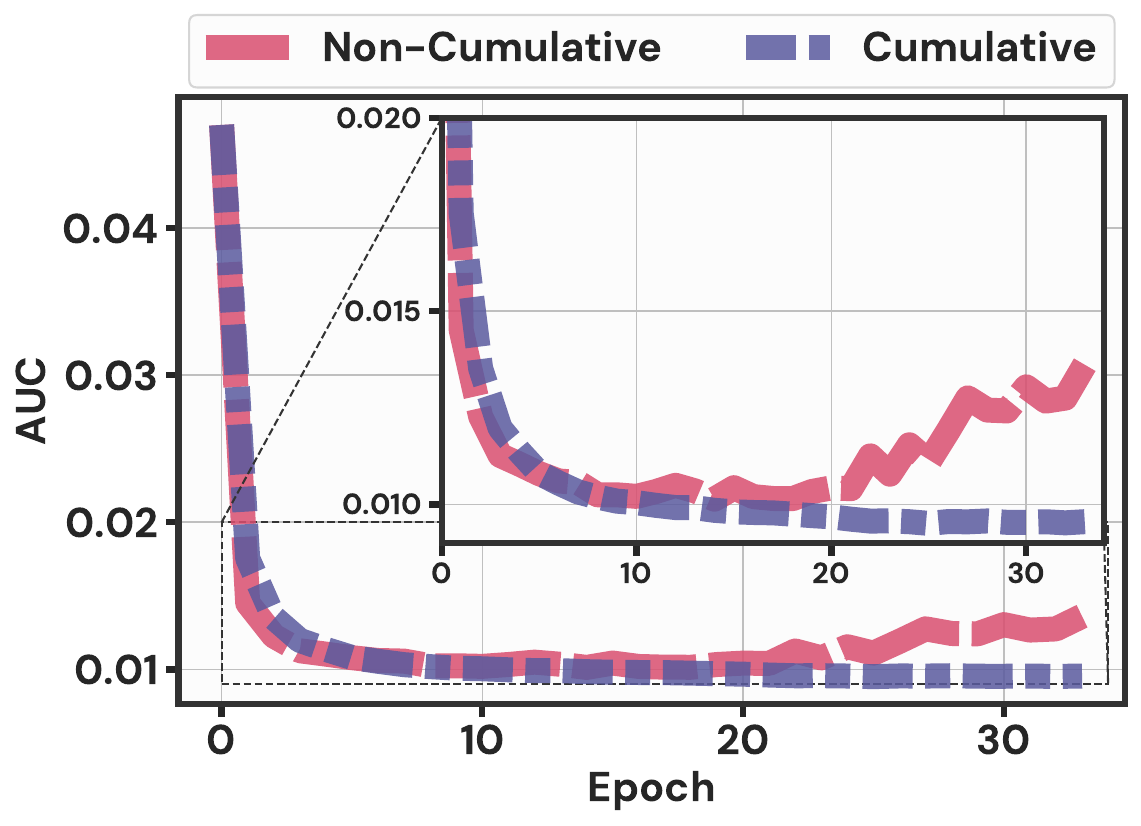}
  }
  \par
  \subfloat[%
    Non-cumulative target distribution over epochs. Scores drift rightward and compress, reducing sample differentiation.%
    \label{fig:dist-non-cumulative}
  ]{%
    \includegraphics[width=0.46\linewidth]{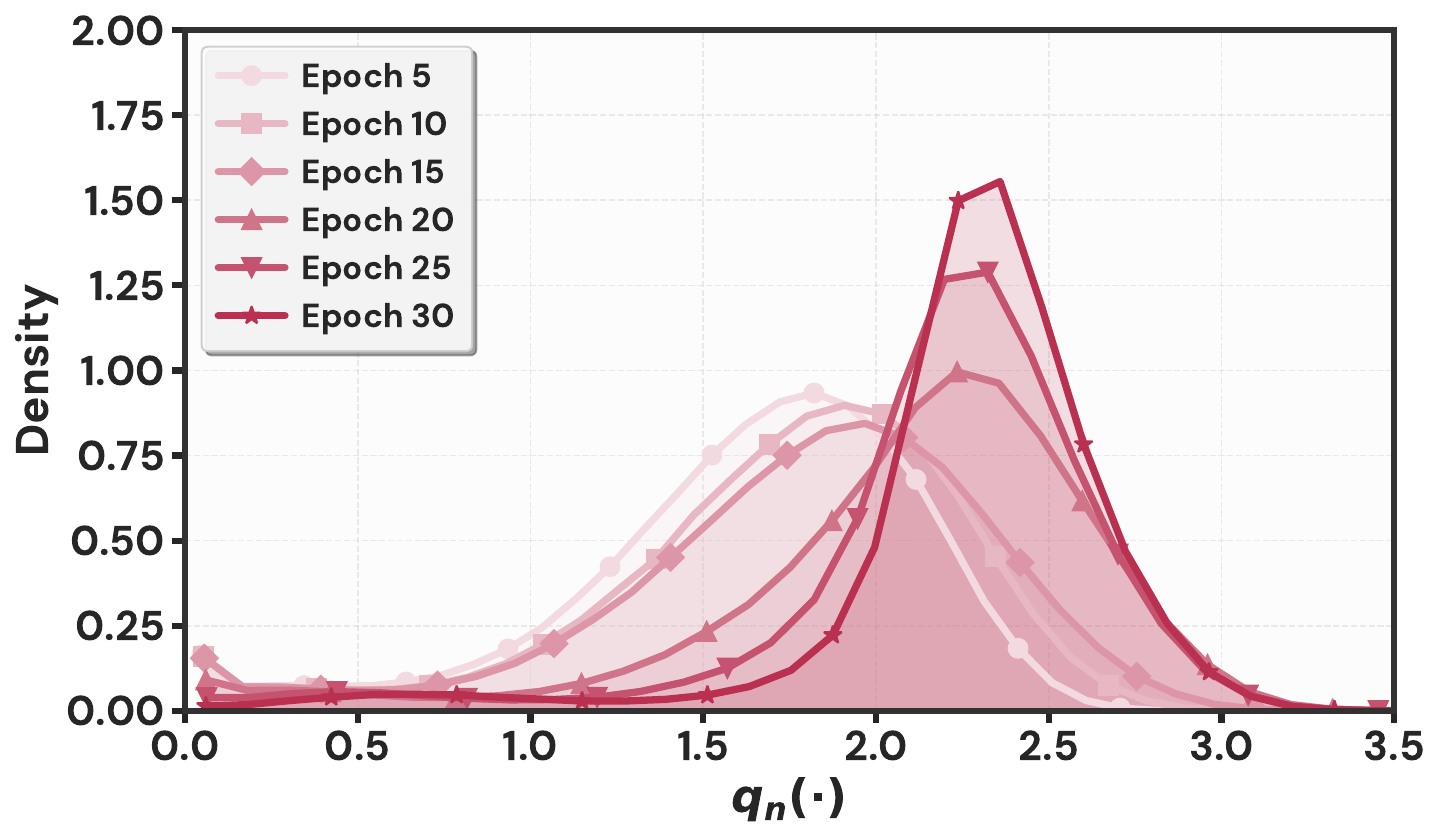}
  }
  \hfill
  \subfloat[%
    Cumulative target distributions remain more stable and well-spread across epochs, preserving sample differentiation.%
    \label{fig:dist-cumulative}
  ]{%
    \includegraphics[width=0.46\linewidth]{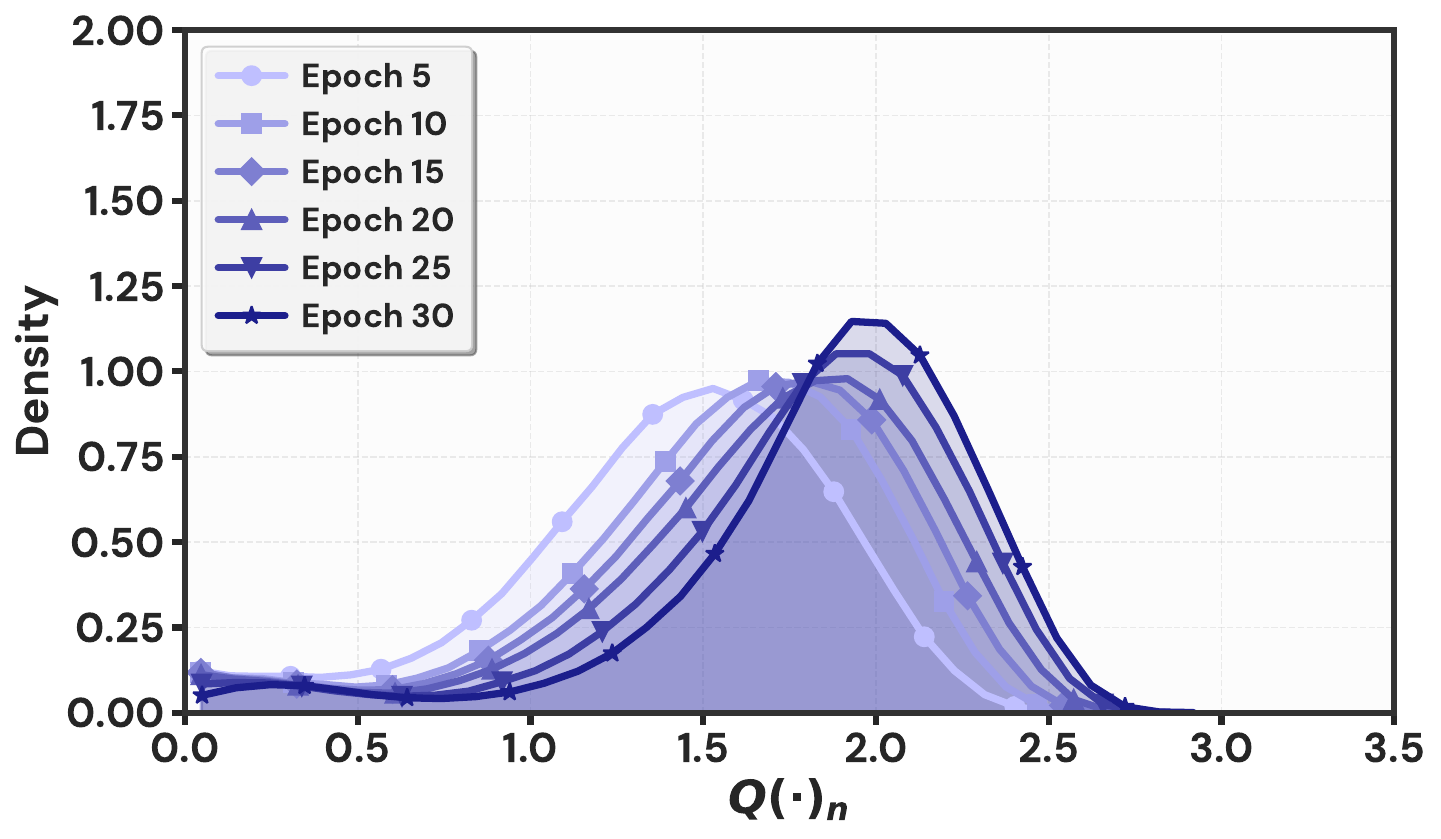}
  }
  \caption{Comparison of cumulative and non-cumulative quality estimation. (a) EDC curves and (b) AUC-EDC trends show cumulative averaging improves estimates consistently, while non-cumulative degrades after epoch 20. (c, d) Quality target distributions highlight non-cumulative compression versus cumulative stability.}
  \label{fig:non_cum_vs_cum}
  \vspace{-5mm}
\end{figure}

\begin{table*}[t]
\centering
\caption{Comparison of AUC-EDC and pAUC-EDC (lower is better) for non-cumulative and cumulative quality estimation across models and benchmarks (FMR: $1e-3$, $1e-4$) as well as a frozen backbone with non-cumulative targets. All models (frozen, non-cumulative, and cumulative) are trained with ArcFace loss, ResNet50 architecture, CASIA-WebFace training dataset. Best results are \textbf{bolded}. Cumulative outperforms non-cumulative on most benchmarks; the consistent exceptions are AgeDB-30 and CALFW, where non-cumulative performs better. On XQLFW, the ordering (averaged across FR models; see the Average rows) is frozen best, followed by cumulative, then non-cumulative, whereas frozen is generally the weakest variant on all other benchmarks. All values are multiplied by 1000 for easy inspection.}\vspace{-1mm}
The \textit{Average} rows report each variant's value averaged across the six FR models, and the last column (\textit{Avg.\ norm.}) reports each row's values averaged across the seven benchmarks and both FMR thresholds after normalizing each benchmark-FMR column by the best variant in its comparison group (1.0 = best everywhere; lower is better). Cumulative achieves the best overall performance across benchmarks and FMR thresholds, attaining the lowest \textit{Avg.\ norm.} value for every FR model and on average, on both AUC-EDC and pAUC-EDC.
\label{tab:ablation_arcface_webface}
\resizebox{\textwidth}{!}{\begin{tabular}{|c|c|c|cc|cc|cc|cc|cc|cc|cc|c|}
\hline
 \multirow{3}{*}{Metric} & \multirow{3}{*}{FR} & \multirow{3}{*}{Quality Metric}
& \multicolumn{2}{c|}{Adience} & \multicolumn{2}{c|}{AgeDB-30} & \multicolumn{2}{c|}{CFP-FP}
& \multicolumn{2}{c|}{LFW} & \multicolumn{2}{c|}{CALFW} & \multicolumn{2}{c|}{CPLFW}
& \multicolumn{2}{c|}{XQLFW} & \multirow{3}{*}{\text{\shortstack{Avg.\\norm.}}} \\
 & &
 & \multicolumn{2}{c|}{\cite{Adience}} & \multicolumn{2}{c|}{\cite{agedb}} & \multicolumn{2}{c|}{\cite{cfp-fp}}
 & \multicolumn{2}{c|}{\cite{LFWTech}} & \multicolumn{2}{c|}{\cite{CALFW}} & \multicolumn{2}{c|}{\cite{CPLFWTech}}
 & \multicolumn{2}{c|}{\cite{XQLFW}} & \\
\cline{4-17}
 & & & $1e{-3}$ & $1e{-4}$ & $1e{-3}$ & $1e{-4}$ & $1e{-3}$ & $1e{-4}$ & $1e{-3}$ & $1e{-4}$ & $1e{-3}$ & $1e{-4}$ & $1e{-3}$ & $1e{-4}$ & $1e{-3}$ & $1e{-4}$ & \\
\hline \hline
\multirow{21}{*}{AUC-EDC} & \multirow{3}{*}{ArcFace \cite{deng2019arcface}} & frozen & 33.335 & 83.241 & 20.781 & 25.034 & 20.833 & 27.905 & 1.981 & 2.485 & 58.237 & 62.910 & 57.322 & 82.565 & \textbf{213.641} & \textbf{240.517} & \text{1.547} \\
 &  & non-cumulative & 23.407 & 52.503 & \textbf{16.533} & \textbf{20.650} & 10.162 & 15.686 & 2.062 & 2.529 & \textbf{52.835} & \textbf{55.625} & 42.358 & 60.225 & 254.726 & 304.557 & \text{1.184} \\
 &  & cumulative & \textbf{21.295} & \textbf{42.531} & 16.879 & 22.410 & \textbf{6.486} & \textbf{9.244} & \textbf{1.570} & \textbf{2.230} & 58.807 & 62.783 & \textbf{41.057} & \textbf{57.567} & 231.964 & 259.898 & \text{\textbf{1.037}} \\ \cline{2-18}
 & \multirow{3}{*}{ElasticFace \cite{elasticface}} & frozen & 36.298 & 70.997 & 20.526 & 22.399 & 13.869 & 18.745 & 1.955 & 2.485 & 56.377 & 58.417 & 43.612 & 62.091 & \textbf{186.640} & \textbf{225.340} & \text{1.461} \\
 &  & non-cumulative & 25.979 & 48.907 & 16.942 & 18.013 & 7.582 & 11.851 & 1.442 & 2.062 & \textbf{51.185} & \textbf{52.779} & 39.750 & 53.091 & 226.792 & 302.551 & \text{1.166} \\
 &  & cumulative & \textbf{22.682} & \textbf{39.587} & \textbf{16.688} & \textbf{17.823} & \textbf{5.560} & \textbf{7.545} & \textbf{1.109} & \textbf{1.922} & 57.150 & 58.579 & \textbf{39.060} & \textbf{51.272} & 204.030 & 249.978 & \text{\textbf{1.031}} \\ \cline{2-18}
 & \multirow{3}{*}{MagFace \cite{MagFace}} & frozen & 34.540 & 78.648 & 22.082 & \textbf{28.350} & 27.453 & 32.762 & 2.256 & 3.032 & 57.207 & 58.898 & 51.773 & 120.113 & \textbf{241.775} & \textbf{277.810} & \text{1.560} \\
 &  & non-cumulative & 23.647 & 52.045 & \textbf{18.694} & 30.821 & 12.606 & 14.723 & 2.146 & 2.683 & \textbf{51.981} & \textbf{53.047} & 44.919 & 87.496 & 281.443 & 328.661 & \text{1.174} \\
 &  & cumulative & \textbf{21.605} & \textbf{42.378} & 19.017 & 33.299 & \textbf{7.832} & \textbf{11.075} & \textbf{1.578} & \textbf{2.189} & 57.975 & 59.150 & \textbf{41.753} & \textbf{81.590} & 265.804 & 320.926 & \text{\textbf{1.048}} \\ \cline{2-18}
 & \multirow{3}{*}{CurricularFace \cite{curricularFace}} & frozen & 29.950 & 71.524 & 20.424 & 25.872 & 21.887 & 28.336 & 1.981 & 2.650 & 57.188 & 60.268 & 41.033 & 64.274 & \textbf{180.657} & \textbf{215.269} & \text{1.461} \\
 &  & non-cumulative & 21.851 & 46.802 & 18.334 & 21.500 & 11.540 & 16.709 & 2.062 & 2.529 & \textbf{51.814} & \textbf{54.149} & \textbf{36.398} & 53.268 & 215.303 & 246.469 & \text{1.165} \\
 &  & cumulative & \textbf{20.274} & \textbf{36.740} & \textbf{18.031} & \textbf{21.229} & \textbf{7.375} & \textbf{10.285} & \textbf{1.634} & \textbf{2.230} & 57.430 & 60.509 & 36.791 & \textbf{53.024} & 204.355 & 239.521 & \text{\textbf{1.034}} \\ \cline{2-18}
 & \multirow{3}{*}{TransFace \cite{transface}} & frozen & 33.707 & 68.929 & 20.811 & 24.685 & 10.482 & 16.110 & 1.965 & 2.424 & 59.090 & 464.584 & 34.994 & 46.890 & \textbf{69.294} & \textbf{93.212} & \text{1.412} \\
 &  & non-cumulative & 22.517 & 39.767 & \textbf{15.743} & \textbf{18.092} & 5.506 & 8.597 & 2.045 & 2.504 & \textbf{52.110} & \textbf{456.428} & \textbf{34.111} & \textbf{45.907} & 90.642 & 117.726 & \text{1.135} \\
 &  & cumulative & \textbf{20.823} & \textbf{33.604} & 16.467 & 18.680 & \textbf{4.446} & \textbf{6.650} & \textbf{1.553} & \textbf{2.085} & 58.434 & 565.954 & 35.297 & 47.288 & 77.817 & 109.890 & \text{\textbf{1.058}} \\ \cline{2-18}
 & \multirow{3}{*}{SwinFace \cite{swinface}} & frozen & 50.564 & 128.186 & 23.977 & 28.554 & 47.800 & 146.308 & 2.404 & 3.058 & 56.860 & 58.081 & 87.223 & 130.785 & \textbf{162.007} & \textbf{212.233} & \text{1.456} \\
 &  & non-cumulative & 32.844 & 78.940 & \textbf{19.393} & \textbf{22.898} & 25.584 & \textbf{84.612} & 2.078 & 2.691 & \textbf{50.364} & \textbf{54.585} & 63.878 & 89.421 & 198.152 & 251.363 & \text{1.121} \\
 &  & cumulative & \textbf{27.350} & \textbf{62.377} & 21.021 & 25.854 & \textbf{22.283} & 108.772 & \textbf{1.604} & \textbf{2.233} & 56.743 & 62.658 & \textbf{59.125} & \textbf{82.312} & 171.105 & 243.041 & \text{\textbf{1.070}} \\
 \cline{2-18}
 & \multirow{3}{*}{\text{Average}} & \text{frozen} & \text{36.399} & \text{83.587} & \text{21.433} & \text{25.816} & \text{23.721} & \text{45.028} & \text{2.091} & \text{2.689} & \text{57.493} & \text{127.193} & \text{52.660} & \text{84.453} & \text{\textbf{175.669}} & \text{\textbf{210.730}} & \text{1.413} \\
 &  & \text{non-cumulative} & \text{25.041} & \text{53.161} & \text{\textbf{17.606}} & \text{\textbf{21.996}} & \text{12.163} & \text{\textbf{25.363}} & \text{1.972} & \text{2.500} & \text{\textbf{51.715}} & \text{\textbf{121.102}} & \text{43.569} & \text{64.901} & \text{211.176} & \text{258.554} & \text{1.121} \\
 &  & \text{cumulative} & \text{\textbf{22.338}} & \text{\textbf{42.869}} & \text{18.017} & \text{23.216} & \text{\textbf{8.997}} & \text{25.595} & \text{\textbf{1.508}} & \text{\textbf{2.148}} & \text{57.757} & \text{144.939} & \text{\textbf{42.181}} & \text{\textbf{62.176}} & \text{192.513} & \text{237.209} & \text{\textbf{1.045}} \\
\hline \hline
\multirow{21}{*}{pAUC-EDC} & \multirow{3}{*}{ArcFace \cite{deng2019arcface}} & frozen & 13.063 & 32.980 & 8.423 & 11.422 & 6.858 & 11.877 & 0.702 & 0.864 & 22.940 & 25.581 & 22.498 & 36.782 & \textbf{143.146} & \textbf{157.559} & \text{1.209} \\
 &  & non-cumulative & 11.715 & 29.379 & \textbf{7.910} & \textbf{11.403} & 5.009 & 9.204 & 0.805 & 0.931 & \textbf{21.367} & \textbf{23.124} & 22.664 & 35.959 & 155.245 & 188.752 & \text{1.142} \\
 &  & cumulative & \textbf{10.511} & \textbf{26.547} & 8.342 & 12.519 & \textbf{3.709} & \textbf{5.864} & \textbf{0.588} & \textbf{0.822} & 22.131 & 24.525 & \textbf{21.871} & \textbf{35.208} & 147.636 & 161.204 & \text{\textbf{1.022}} \\ \cline{2-18}
 & \multirow{3}{*}{ElasticFace \cite{elasticface}} & frozen & 14.336 & 29.034 & 8.414 & 9.306 & 5.487 & 8.376 & 0.676 & 0.864 & 21.761 & 22.936 & 21.523 & 33.859 & \textbf{133.352} & \textbf{153.448} & \text{1.192} \\
 &  & non-cumulative & 13.043 & 25.965 & \textbf{7.749} & \textbf{8.213} & 4.455 & 6.945 & 0.688 & 0.931 & \textbf{20.437} & \textbf{20.966} & 21.379 & 32.451 & 142.825 & 182.490 & \text{1.129} \\
 &  & cumulative & \textbf{11.679} & \textbf{23.546} & 7.879 & 8.434 & \textbf{3.503} & \textbf{4.892} & \textbf{0.471} & \textbf{0.822} & 21.442 & 21.948 & \textbf{20.973} & \textbf{31.922} & 135.770 & 159.046 & \text{\textbf{1.014}} \\ \cline{2-18}
 & \multirow{3}{*}{MagFace \cite{MagFace}} & frozen & 13.328 & 32.788 & 9.381 & \textbf{14.813} & 9.434 & 13.259 & 0.727 & 1.070 & 22.432 & 23.350 & 24.843 & 68.376 & \textbf{153.361} & \textbf{173.775} & \text{1.214} \\
 &  & non-cumulative & 11.877 & 28.786 & \textbf{8.305} & 17.189 & 7.066 & 8.804 & 0.889 & 1.085 & \textbf{20.892} & \textbf{21.514} & 24.610 & 58.882 & 161.753 & 182.126 & \text{1.128} \\
 &  & cumulative & \textbf{10.841} & \textbf{26.025} & 8.694 & 18.922 & \textbf{5.236} & \textbf{8.094} & \textbf{0.596} & \textbf{0.792} & 21.990 & 22.386 & \textbf{24.247} & \textbf{58.189} & 163.947 & 193.047 & \text{\textbf{1.043}} \\ \cline{2-18}
 & \multirow{3}{*}{CurricularFace \cite{curricularFace}} & frozen & 11.386 & 27.017 & 9.114 & 11.654 & 7.067 & 10.230 & 0.702 & 0.937 & 22.422 & 24.161 & 19.225 & 35.952 & \textbf{132.180} & \textbf{150.382} & \text{1.176} \\
 &  & non-cumulative & 10.179 & 24.340 & \textbf{8.324} & \textbf{10.146} & 5.580 & 8.393 & 0.805 & 0.931 & \textbf{20.564} & \textbf{22.013} & 19.443 & 34.322 & 139.079 & 155.219 & \text{1.096} \\
 &  & cumulative & \textbf{9.422} & \textbf{21.721} & 8.535 & 10.488 & \textbf{4.304} & \textbf{6.289} & \textbf{0.641} & \textbf{0.822} & 21.693 & 23.279 & \textbf{19.076} & \textbf{33.470} & 133.966 & 152.651 & \text{\textbf{1.014}} \\ \cline{2-18}
 & \multirow{3}{*}{TransFace \cite{transface}} & frozen & 12.144 & 25.793 & 7.664 & 9.503 & 3.997 & 5.959 & 0.685 & 0.803 & 22.521 & 254.983 & 16.763 & 28.080 & 54.059 & \textbf{74.149} & \text{1.163} \\
 &  & non-cumulative & 10.718 & 21.543 & \textbf{7.044} & \textbf{8.295} & 3.087 & 4.731 & 0.789 & 0.906 & \textbf{20.657} & 253.444 & 17.087 & 28.307 & 62.373 & 79.862 & \text{1.109} \\
 &  & cumulative & \textbf{9.896} & \textbf{19.418} & 7.550 & 8.675 & \textbf{2.588} & \textbf{4.066} & \textbf{0.571} & \textbf{0.688} & 21.775 & \textbf{252.483} & \textbf{16.556} & \textbf{28.070} & \textbf{53.721} & 77.714 & \text{\textbf{1.016}} \\ \cline{2-18}
 & \multirow{3}{*}{SwinFace \cite{swinface}} & frozen & 20.047 & 49.345 & 10.669 & 13.581 & 18.906 & 74.479 & 0.826 & 1.096 & 22.415 & \textbf{23.205} & 36.601 & 54.875 & 117.684 & \textbf{159.068} & \text{1.174} \\
 &  & non-cumulative & 17.647 & 44.439 & \textbf{9.066} & \textbf{11.402} & 15.292 & 65.581 & 0.822 & 1.093 & \textbf{20.426} & 24.204 & 35.628 & 51.720 & 132.657 & 166.825 & \text{1.099} \\
 &  & cumulative & \textbf{15.418} & \textbf{40.226} & 10.052 & 12.519 & \textbf{12.981} & \textbf{63.535} & \textbf{0.622} & \textbf{0.836} & 21.385 & 26.856 & \textbf{34.375} & \textbf{50.927} & \textbf{114.485} & 161.526 & \text{\textbf{1.030}} \\
\cline{2-18}
 & \multirow{3}{*}{\text{Average}} & \text{frozen} & \text{14.051} & \text{32.826} & \text{8.944} & \text{11.713} & \text{8.625} & \text{20.697} & \text{0.719} & \text{0.939} & \text{22.415} & \text{62.369} & \text{23.575} & \text{42.987} & \text{\textbf{122.297}} & \text{\textbf{144.730}} & \text{1.160} \\
 &  & \text{non-cumulative} & \text{12.530} & \text{29.075} & \text{\textbf{8.066}} & \text{\textbf{11.108}} & \text{6.748} & \text{17.276} & \text{0.800} & \text{0.979} & \text{\textbf{20.724}} & \text{\textbf{60.877}} & \text{23.468} & \text{40.273} & \text{132.322} & \text{159.213} & \text{1.101} \\
 &  & \text{cumulative} & \text{\textbf{11.294}} & \text{\textbf{26.247}} & \text{8.509} & \text{11.926} & \text{\textbf{5.387}} & \text{\textbf{15.457}} & \text{\textbf{0.581}} & \text{\textbf{0.797}} & \text{21.736} & \text{61.913} & \text{\textbf{22.850}} & \text{\textbf{39.631}} & \text{124.921} & \text{150.865} & \text{\textbf{1.018}} \\
\hline
\end{tabular}}
\vspace{-4mm}
\end{table*}

\begin{table*}[t]
\centering
\caption{Comparison of AUC-EDC and pAUC-EDC (lower is better) under the same experimental setting as Tab. \ref{tab:ablation_arcface_webface}, but with CurricularFace loss instead of ArcFace.}
\label{tab:ablation_curricularface_webface}
\resizebox{\textwidth}{!}{\begin{tabular}{|c|c|c|cc|cc|cc|cc|cc|cc|cc|c|}
\hline
 \multirow{3}{*}{Metric} & \multirow{3}{*}{FR} & \multirow{3}{*}{Quality Metric}
& \multicolumn{2}{c|}{Adience} & \multicolumn{2}{c|}{AgeDB-30} & \multicolumn{2}{c|}{CFP-FP}
& \multicolumn{2}{c|}{LFW} & \multicolumn{2}{c|}{CALFW} & \multicolumn{2}{c|}{CPLFW}
& \multicolumn{2}{c|}{XQLFW} & \multirow{3}{*}{\text{\shortstack{Avg.\\norm.}}} \\
 & &
 & \multicolumn{2}{c|}{\cite{Adience}} & \multicolumn{2}{c|}{\cite{agedb}} & \multicolumn{2}{c|}{\cite{cfp-fp}}
 & \multicolumn{2}{c|}{\cite{LFWTech}} & \multicolumn{2}{c|}{\cite{CALFW}} & \multicolumn{2}{c|}{\cite{CPLFWTech}}
 & \multicolumn{2}{c|}{\cite{XQLFW}} & \\
\cline{4-17}
 & & & $1e{-3}$ & $1e{-4}$ & $1e{-3}$ & $1e{-4}$ & $1e{-3}$ & $1e{-4}$ & $1e{-3}$ & $1e{-4}$ & $1e{-3}$ & $1e{-4}$ & $1e{-3}$ & $1e{-4}$ & $1e{-3}$ & $1e{-4}$ & \\
\hline  \hline
\multirow{14}{*}{AUC-EDC} & \multirow{2}{*}{ArcFace \cite{deng2019arcface}} & non-cumulative & 23.938 & 56.259 & 17.593 & 22.133 & 6.792 & 10.693 & 1.688 & 2.155 & 57.384 & 61.981 & 39.313 & 56.546 & 260.345 & 312.863 & \text{1.052} \\
 &  & cumulative & 22.070 & 44.323 & 17.897 & 22.728 & 6.791 & 10.743 & 1.953 & 2.573 & 60.759 & 63.547 & 41.281 & 57.975 & 218.855 & 262.735 & \text{\textbf{1.040}} \\ \cline{2-18}
 & \multirow{2}{*}{ElasticFace \cite{elasticface}} & non-cumulative & 26.439 & 49.392 & 17.313 & 18.149 & 7.707 & 11.052 & 1.402 & 2.074 & 55.857 & 57.723 & 38.517 & 51.568 & 241.643 & 263.596 & \text{1.083} \\
 &  & cumulative & 23.646 & 40.996 & 17.973 & 18.768 & 5.867 & 8.173 & 1.630 & 2.457 & 59.054 & 60.379 & 39.146 & 51.597 & 205.884 & 277.024 & \text{\textbf{1.042}} \\ \cline{2-18}
 & \multirow{2}{*}{MagFace \cite{MagFace}} & non-cumulative & 24.360 & 56.304 & 19.064 & 31.340 & 10.384 & 14.833 & 1.840 & 2.359 & 56.578 & 57.619 & 41.334 & 90.054 & 314.435 & 360.528 & \text{1.104} \\
 &  & cumulative & 22.569 & 44.961 & 19.856 & 30.845 & 8.410 & 10.332 & 2.017 & 2.688 & 60.317 & 61.536 & 42.214 & 82.313 & 253.009 & 326.348 & \text{\textbf{1.031}} \\ \cline{2-18}
 & \multirow{2}{*}{CurricularFace \cite{curricularFace}} & non-cumulative & 22.344 & 47.814 & 18.744 & 21.328 & 8.845 & 13.727 & 1.688 & 2.155 & 56.748 & 59.930 & 35.956 & 53.613 & 223.808 & 246.655 & \text{1.075} \\
 &  & cumulative & 20.972 & 39.157 & 18.443 & 22.349 & 7.456 & 11.001 & 1.974 & 2.573 & 60.193 & 62.258 & 36.056 & 52.809 & 181.650 & 232.594 & \text{\textbf{1.037}} \\ \cline{2-18}
 & \multirow{2}{*}{TransFace \cite{transface}} & non-cumulative & 23.518 & 42.535 & 16.493 & 18.453 & 4.894 & 7.850 & 1.671 & 2.130 & 89.819 & 703.441 & 34.328 & 46.240 & 86.241 & 112.670 & \text{1.223} \\
 &  & cumulative & 21.821 & 35.608 & 17.465 & 19.920 & 4.900 & 7.149 & 1.936 & 2.473 & 59.590 & 235.642 & 34.372 & 45.952 & 75.502 & 101.402 & \text{\textbf{1.033}} \\ \cline{2-18}
 & \multirow{2}{*}{SwinFace \cite{swinface}} & non-cumulative & 33.254 & 81.283 & 18.530 & 23.991 & 26.913 & 106.746 & 1.866 & 2.546 & 55.328 & 57.665 & 58.698 & 84.181 & 196.414 & 250.585 & \text{1.070} \\
 &  & cumulative & 29.123 & 65.360 & 20.820 & 25.735 & 21.405 & 102.007 & 2.045 & 2.772 & 59.605 & 61.924 & 58.378 & 82.631 & 173.533 & 221.822 & \text{\textbf{1.038}} \\ \cline{2-18}
 & \multirow{2}{*}{\text{Average}} & \text{non-cumulative} & \text{25.642} & \text{55.598} & \text{\textbf{17.956}} & \text{\textbf{22.566}} & \text{10.922} & \text{27.483} & \text{\textbf{1.692}} & \text{\textbf{2.236}} & \text{61.952} & \text{166.393} & \text{\textbf{41.358}} & \text{63.700} & \text{220.481} & \text{257.816} & \text{1.129} \\
 &  & \text{cumulative} & \text{\textbf{23.367}} & \text{\textbf{45.067}} & \text{18.742} & \text{23.391} & \text{\textbf{9.138}} & \text{\textbf{24.901}} & \text{1.926} & \text{2.589} & \text{\textbf{59.920}} & \text{\textbf{90.881}} & \text{41.908} & \text{\textbf{62.213}} & \text{\textbf{184.739}} & \text{\textbf{236.987}} & \text{\textbf{1.028}} \\
\hline \hline
\multirow{14}{*}{pAUC-EDC} & \multirow{2}{*}{ArcFace \cite{deng2019arcface}} & non-cumulative & 11.556 & 29.233 & 8.021 & 11.692 & 4.548 & 7.641 & 0.798 & 0.923 & 21.190 & 23.559 & 22.482 & 35.916 & 154.570 & 183.714 & \text{1.045} \\
 &  & cumulative & 10.909 & 27.330 & 8.438 & 12.754 & 3.937 & 7.189 & 0.827 & 1.028 & 22.097 & 24.097 & 21.697 & 34.497 & 139.762 & 167.165 & \text{\textbf{1.026}} \\ \cline{2-18}
 & \multirow{2}{*}{ElasticFace \cite{elasticface}} & non-cumulative & 13.019 & 25.954 & 7.707 & 8.108 & 4.325 & 6.309 & 0.682 & 0.923 & 20.056 & 21.128 & 21.336 & 33.223 & 148.405 & 159.043 & \text{1.047} \\
 &  & cumulative & 12.202 & 24.048 & 8.132 & 8.592 & 3.642 & 5.389 & 0.710 & 1.028 & 20.929 & 21.661 & 20.485 & 31.341 & 140.463 & 181.493 & \text{\textbf{1.034}} \\ \cline{2-18}
 & \multirow{2}{*}{MagFace \cite{MagFace}} & non-cumulative & 11.843 & 29.948 & 8.356 & 18.076 & 6.536 & 10.508 & 0.914 & 1.091 & 20.982 & 21.477 & 24.591 & 66.731 & 176.116 & 193.452 & \text{1.096} \\
 &  & cumulative & 11.199 & 27.500 & 8.646 & 18.635 & 5.438 & 6.974 & 0.890 & 1.143 & 21.736 & 22.285 & 23.458 & 56.970 & 151.631 & 180.090 & \text{\textbf{1.013}} \\ \cline{2-18}
 & \multirow{2}{*}{CurricularFace \cite{curricularFace}} & non-cumulative & 10.252 & 24.431 & 8.389 & 10.124 & 5.172 & 7.998 & 0.798 & 0.923 & 20.655 & 22.378 & 19.434 & 35.640 & 132.671 & 146.162 & \text{1.049} \\
 &  & cumulative & 9.722 & 22.892 & 8.481 & 10.774 & 4.348 & 6.923 & 0.847 & 1.028 & 21.329 & 22.723 & 18.342 & 32.920 & 123.580 & 149.437 & \text{\textbf{1.023}} \\ \cline{2-18}
 & \multirow{2}{*}{TransFace \cite{transface}} & non-cumulative & 10.776 & 22.293 & 7.168 & 8.341 & 2.994 & 4.782 & 0.781 & 0.898 & 20.869 & 252.980 & 16.879 & 28.228 & 57.275 & 76.680 & \text{1.053} \\
 &  & cumulative & 10.248 & 20.300 & 7.803 & 9.050 & 2.876 & 4.364 & 0.810 & 0.927 & 21.413 & 195.848 & 16.431 & 27.171 & 54.256 & 73.290 & \text{\textbf{1.019}} \\ \cline{2-18}
 & \multirow{2}{*}{SwinFace \cite{swinface}} & non-cumulative & 17.749 & 44.404 & 9.203 & 11.893 & 14.737 & 65.154 & 0.940 & 1.242 & 20.533 & 22.429 & 35.571 & 52.873 & 123.755 & 156.253 & \text{1.041} \\
 &  & cumulative & 16.294 & 41.279 & 10.145 & 12.838 & 13.874 & 63.159 & 0.919 & 1.227 & 21.345 & 23.169 & 33.615 & 50.948 & 115.141 & 140.407 & \text{\textbf{1.018}} \\
\cline{2-18}
 & \multirow{2}{*}{\text{Average}} & \text{non-cumulative} & \text{12.532} & \text{29.377} & \text{\textbf{8.141}} & \text{\textbf{11.373}} & \text{6.385} & \text{17.065} & \text{\textbf{0.819}} & \text{\textbf{1.000}} & \text{\textbf{20.714}} & \text{60.658} & \text{23.382} & \text{42.102} & \text{132.132} & \text{152.551} & \text{1.056} \\
 &  & \text{cumulative} & \text{\textbf{11.762}} & \text{\textbf{27.225}} & \text{8.607} & \text{12.107} & \text{\textbf{5.686}} & \text{\textbf{15.666}} & \text{0.834} & \text{1.063} & \text{21.475} & \text{\textbf{51.630}} & \text{\textbf{22.338}} & \text{\textbf{38.975}} & \text{\textbf{120.806}} & \text{\textbf{148.647}} & \text{\textbf{1.017}} \\
\hline
\end{tabular}}
\vspace{-3mm}
\end{table*}

\begin{figure}[t]
    \centering
    \setlength{\tabcolsep}{1pt}
    \renewcommand{\arraystretch}{1.1}
    \resizebox{\linewidth}{!}{\includegraphics[width=\linewidth]{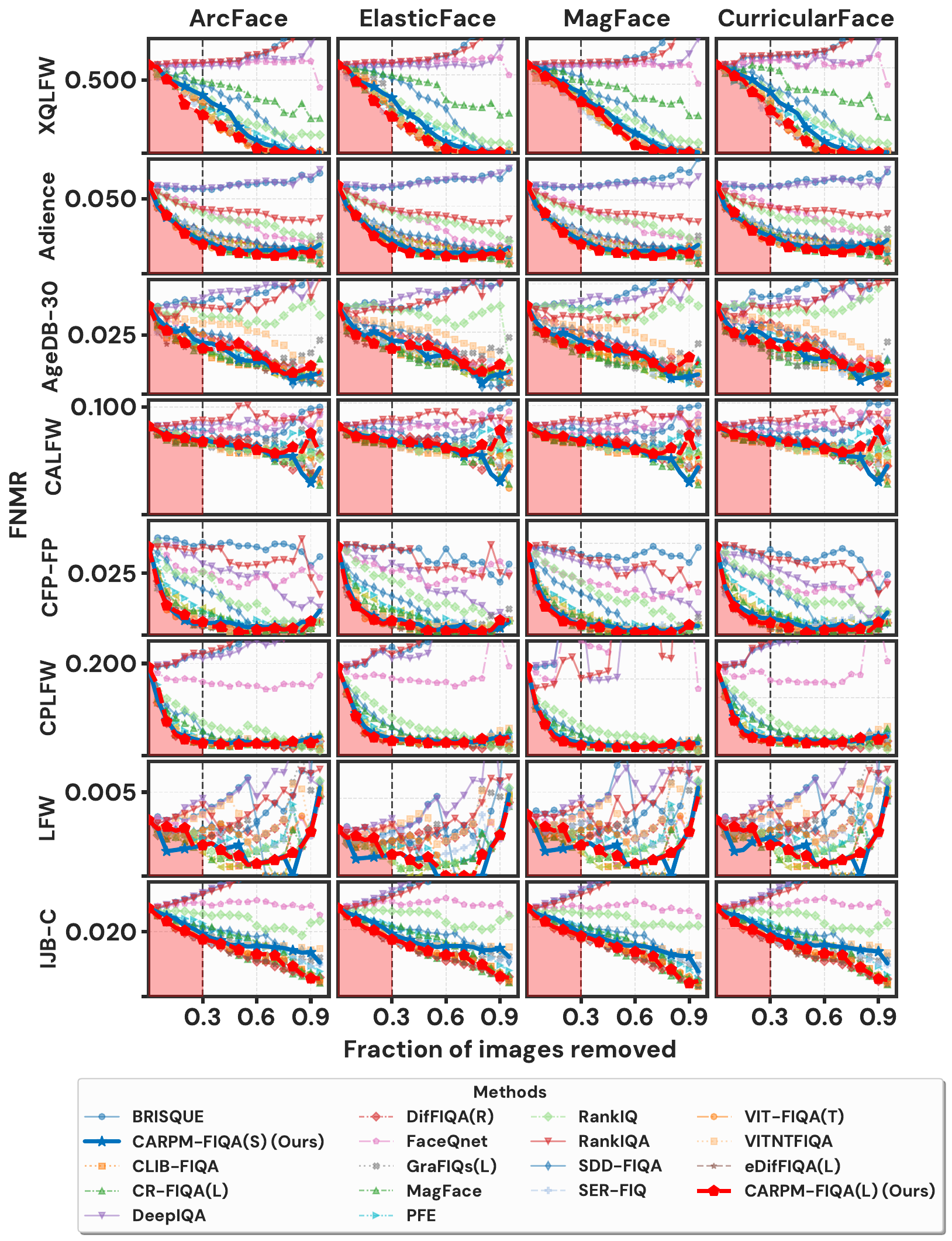}}
    \caption{Error-versus-Discard Characteristic (EDC) curves for FNMR@FMR=$1e-3$ of our proposed method \ourmethod in comparison to SOTA. Results shown on eight benchmark datasets: LFW \cite{LFWTech}, AgeDB-30 \cite{agedb}, CFP-FP \cite{cfp-fp}, CALFW \cite{CALFW}, Adience \cite{Adience}, CPLFW \cite{CPLFWTech}, XQLFW \cite{XQLFW}, and IJB-C \cite{ijbc}, using ArcFace \cite{deng2019arcface}, ElasticFace \cite{elasticface}, MagFace \cite{MagFace}, and CurricularFace \cite{curricularFace} FR models.}
    \label{fig:fnmr3}
    \vspace{-7mm}
\end{figure}

\begin{table}[t]
\centering
\caption{Influence of the accumulation window: window length $k$ (of $T$ training epochs) at which the growing first-$k$ window reaches Spearman rank correlation $\rho \geq 0.95$/$0.99$ against the full-trajectory target $Q_T$, and Spearman correlation with $Q_T$ of first-$k$ vs.\ last-$k$ windows at a matched relative window length ($k \approx 24\%$ of $T$: $k{=}8$ of 34 for the small protocol, $k{=}4$ of 18 for the large).}
\label{tab:window_main}
\resizebox{\linewidth}{!}{%
\begin{tabular}{lccccc}
\toprule
Configuration & $T$ & $k$@$\rho{=}0.95$ & $k$@$\rho{=}0.99$ & first-$k$ & last-$k$ \\
\midrule
ArcFace, cumulative, small            & 34 & 10 & 24 & \textbf{0.939} & 0.845 \\
ArcFace, non-cumulative, small        & 34 & 11 & 24 & \textbf{0.938} & 0.857 \\
CurricularFace, cumulative, small     & 34 & 11 & 24 & \textbf{0.935} & 0.854 \\
CurricularFace, non-cumulative, small & 34 & 11 & 24 & \textbf{0.934} & 0.869 \\
ArcFace, cumulative, large            & 18 & 8  & 14 & \textbf{0.876} & 0.841 \\
ArcFace, non-cumulative, large        & 18 & 8  & 14 & \textbf{0.883} & 0.840 \\
CurricularFace, cumulative, large     & 18 & 8  & 14 & \textbf{0.890} & 0.835 \\
CurricularFace, non-cumulative, large & 18 & 8  & 14 & \textbf{0.892} & 0.835 \\
\bottomrule
\end{tabular}}\vspace{-6mm}
\end{table}


\begin{table*}[t]
    \begin{center}
    \caption{The pAUC-EDC achieved by \ourmethod and the SOTA methods under different evaluation settings. The notions of $1e-3$ and $1e-4$ indicate the value of the fixed FMR at which the EDC curves (FNMR vs. reject) were calculated. The results are compared to three IQA and twelve FIQA approaches. All values are multiplied by 1000 for easy inspection.}
    The \textit{Average} row block reports each method's value averaged across the four FR models, and the last column (\textit{Avg.\ norm.}) reports each row's values averaged across the eight benchmarks and both FMR thresholds after normalizing each benchmark-FMR column by its best method (1.0 = matches the best method everywhere; lower is better).
    \label{tab:sota_pauc}
    \resizebox{\textwidth}{!}{
\begin{tabular}{|c|cl|cc|cc|cc|cc|cc|cc|cc|cc|c|}
    \hline
\multirow{3}{*}{FR} & \multicolumn{2}{c|}{\multirow{3}{*}{Method}}
& \multicolumn{2}{c|}{Adience} & \multicolumn{2}{c|}{AgeDB-30} & \multicolumn{2}{c|}{CFP-FP}
& \multicolumn{2}{c|}{LFW} & \multicolumn{2}{c|}{CALFW} & \multicolumn{2}{c|}{CPLFW}
& \multicolumn{2}{c|}{XQLFW} & \multicolumn{2}{c|}{IJB-C} & \multirow{3}{*}{\text{\shortstack{Avg.\\norm.}}} \\
 & &
 & \multicolumn{2}{c|}{\cite{Adience}} & \multicolumn{2}{c|}{\cite{agedb}} & \multicolumn{2}{c|}{\cite{cfp-fp}}
 & \multicolumn{2}{c|}{\cite{LFWTech}} & \multicolumn{2}{c|}{\cite{CALFW}} & \multicolumn{2}{c|}{\cite{CPLFWTech}}
 & \multicolumn{2}{c|}{\cite{XQLFW}}  & \multicolumn{2}{c|}{\cite{ijbc}} & \\ \cline{4-19}
                                                                                    &                                                   &  & $1e{-3}$ & $1e{-4}$ & $1e{-3}$ & $1e{-4}$  & $1e{-3}$ & $1e{-4}$  & $1e{-3}$ & $1e{-4}$  & $1e{-3}$ & $1e{-4}$ & $1e{-3}$ & $1e{-4}$ & $1e{-3}$ & $1e{-4}$ & $1e{-3}$ & $1e{-4}$ & \\ \hline
     \hline
     \multirow{17}{*}{\rotatebox[origin=c]{90}{ArcFace\cite{deng2019arcface}}}      & \multirow{3}{*}{\rotatebox[origin=c]{90}{IQA}}    & \multicolumn{1}{|l|}{BRISQUE\cite{BRISQE_IQA}} &17.1746 &39.8035 &11.5267 &17.4919 &11.3476 &16.2225 &1.0438 &1.2273 &23.9653 &26.8772 &61.1376 &77.2797 &181.8562 &203.0627 &8.7528 &13.7855                                                & \text{1.858} \\
                                                                                    &                                                   & \multicolumn{1}{|l|}{RankIQA\cite{liu2017rankiqa}} &14.6239 &35.4490 &10.5402 &17.0561 &10.5550 &15.9389 &1.0739 &1.2560 &25.2499 &28.2475 &61.6504 &80.0431 &184.0819 &205.2308 &8.6993 &13.6493                                                & \text{1.821} \\
                                                                                    &                                                   & \multicolumn{1}{|l|}{DeepIQA\cite{DEEPIQ_IQA}} &17.3394 &40.4844 &11.3285 &16.8783 &10.3545 &13.6823 &1.1158 &1.2964 &24.9032 &27.6482 &61.0665 &78.9024 &178.4038 &200.0800 &8.7849 &13.9447                                                & \text{1.828} \\
                                                                                    \cline{2-20}
                                                                                    & \multirow{14}{*}{\rotatebox[origin=c]{90}{FIQA}}   & \multicolumn{1}{|l|}{RankIQ\cite{RANKIQ_FIQA}} &14.5717 &35.7920 &10.7000 &17.1308 &8.3235 &13.3008 &0.7481 &0.9286 &22.5220 &25.4321 &34.5667 &48.4412 &148.0803 &172.4149 &7.8981 &12.1997 & \text{1.488} \\
                                                                                    &                                                   & \multicolumn{1}{|l|}{PFE\cite{PFE_FIQA}} &10.7396 &27.0895 &8.2111 &12.6747 &5.8932 &8.9760 &0.7953 &0.9206 &22.2564 &24.3611 &26.6043 &42.7585 &142.4588 &171.9463 &7.4699 &11.0029 & \text{1.245} \\
                                                                                    &                                                   & \multicolumn{1}{|l|}{SER-FIQ\cite{SERFIQ}} &11.6273 &27.4337 &7.7763 &12.2833 &3.7974 &6.3054 &0.8002 &0.9747 &22.0528 &24.2024 &21.5703 &35.0857 &132.3684 &156.8467 &6.5277 &10.0928 & \text{1.126} \\
                                                                                    &                                                   & \multicolumn{1}{|l|}{FaceQnet\cite{hernandez2019faceqnet,faceqnetv1}} &15.2730 &35.4687 &8.8043 &12.7036 &9.0087 &11.4703 &1.0069 &1.1323 &23.3391 &25.7229 &50.8810 &65.2778 &183.1442 &202.2128 &8.5016 &12.6979 & \text{1.602} \\
                                                                                    &                                                   & \multicolumn{1}{|l|}{MagFace\cite{MagFace}} &11.1539 &27.5222 &7.4280 &10.8164 &4.9517 &7.4949 &0.6803 &0.8408 &21.0663 &22.8291 &27.6654 &39.9884 &160.8330 &190.5259 &7.1540 &10.8829 & \text{1.181} \\
                                                                                    &                                                   & \multicolumn{1}{|l|}{SDD-FIQA\cite{SDDFIQA}} &11.8439 &29.7595 &8.6215 &10.1885 &7.8104 &12.4985 &0.7999 &0.9629 &22.3543 &24.2454 &31.1456 &41.9985 &159.1507 &178.9500 &7.2361 &11.0387 & \text{1.344} \\
                                                                                    &                                                   & \multicolumn{1}{|l|}{CR-FIQA(L)\cite{boutros_2023_crfiqa}} &10.9006 &22.3524 &7.6050 &9.9834 &3.6604 &6.1660 &0.8100 &1.0117 &20.9365 &21.9577 &20.3737 &33.2015 &140.0162 &159.1274 &6.5792 &10.1140 & \text{1.078} \\
                                                                                    &                                                   & \multicolumn{1}{|l|}{DifFIQA(R)\cite{10449044}} &11.2264 &29.1213 &9.2688 &13.7286 &3.9312 &6.7074 &0.7889 &0.9299 &21.8013 &24.3666 &20.1850 &33.5504 &137.8221 &158.6149 &6.4824 &9.8720 & \text{1.148} \\
                                                                                    &                                                   & \multicolumn{1}{|l|}{eDifFIQA(L)\cite{babnikTBIOM2024}} &10.2102 &25.4868 &6.8802 &8.8783 &3.5460 &6.0476 &0.7847 &0.9082 &21.0121 &23.4658 &20.0858 &33.3475 &142.3163 &166.8084 &6.4695 &9.7904 & \text{\textbf{1.059}} \\
                                                                                    &                                                   & \multicolumn{1}{|l|}{GraFIQs(L)\cite{grafiqs}} &10.5415 &23.7571 &7.7167 &11.0336 &4.3479 &7.1032 &0.8403 &1.0396 &21.4252 &23.9000 &22.4948 &37.6689 &144.3087 &158.6825 &6.8626 &10.2937 & \text{1.144} \\
                                                                                    &                                                   & \multicolumn{1}{|l|}{CLIB-FIQA\cite{Ou_2024_CVPR}} &10.9312 &27.3190 &7.3872 &9.4355 &4.0696 &6.7689 &0.7903 &0.9149 &21.0636 &23.3403 &20.4309 &33.6266 &137.3993 &150.9315 &6.5955 &9.9566 & \text{1.090} \\
                                                                                    &                                                   & \multicolumn{1}{|l|}{ViT-FIQA(T)\cite{atzori2025vitfiqaassessingfaceimage}} &9.9483 &25.6644 &8.2337 &10.7342 &3.5684 &5.6625 &0.7706 &0.8960 &21.7708 &23.6136 &20.5313 &33.3878 &140.4654 &156.2746 &6.5626 &10.1183 & \text{1.079} \\ \cline{3-20}
                                                                                    &                                                   & \multicolumn{1}{|l|}{\ourmethod(S)} &10.5106 &26.5466 &8.3419 &12.5191 &3.7087 &5.8637 &0.5878 &0.8217 &22.1310 &24.5247 &21.8705 &35.2079 &147.6362 &161.2037 &7.0624 &10.7002 & \text{1.103} \\
                                                                                    &                                                   & \multicolumn{1}{|l|}{\ourmethod(L)} &10.3328 &23.2626 &7.5733 &10.5830 &3.8167 &5.9570 &0.7813 &0.9243 &21.9167 &23.0551 &22.2518 &35.3623 &130.6450 &162.4970 &6.5683 &10.0681 & \text{1.084} \\
     \hline \hline
     \multirow{17}{*}{\rotatebox[origin=c]{90}{ElasticFace\cite{elasticface}}}      & \multirow{3}{*}{\rotatebox[origin=c]{90}{IQA}}    & \multicolumn{1}{|l|}{BRISQUE\cite{BRISQE_IQA}} &19.1770 &36.1957 &10.7763 &11.8841 &9.8852 &12.8922 &0.8696 &1.2273 &23.3285 &24.1776 &52.4823 &127.2291 &169.7469 &196.7763 &8.3813 &13.1581                                                & \text{1.954} \\
                                                                                    &                                                   & \multicolumn{1}{|l|}{RankIQA\cite{liu2017rankiqa}} &16.2262 &32.2490 &10.1585 &11.5864 &9.8444 &12.3928 &0.9008 &1.2560 &24.5531 &25.3420 &52.7986 &131.4940 &171.5528 &199.3418 &8.2855 &12.9129                                                & \text{1.933} \\
                                                                                    &                                                   & \multicolumn{1}{|l|}{DeepIQA\cite{DEEPIQ_IQA}} &19.4208 &36.9403 &10.6133 &11.7936 &8.7117 &11.8020 &0.9451 &1.2964 &24.1481 &24.9591 &52.2817 &130.3347 &165.4807 &194.4742 &8.5887 &13.1338                                                & \text{1.941} \\
                                                                                    \cline{2-20}
                                                                                    & \multirow{14}{*}{\rotatebox[origin=c]{90}{FIQA}}   & \multicolumn{1}{|l|}{RankIQ\cite{RANKIQ_FIQA}} &16.2083 &32.3135 &10.5877 &11.5525 &7.7161 &9.8220 &0.5777 &0.9286 &21.4691 &22.1840 &32.9593 &43.9097 &134.6810 &153.4910 &7.7368 &11.8907 & \text{1.463} \\
                                                                                    &                                                   & \multicolumn{1}{|l|}{PFE\cite{PFE_FIQA}} &11.8038 &23.5341 &7.4372 &7.9875 &5.4140 &7.1738 &0.6783 &0.9206 &21.4824 &22.1103 &23.7355 &70.2523 &133.0683 &160.0044 &7.0621 &10.6943 & \text{1.294} \\
                                                                                    &                                                   & \multicolumn{1}{|l|}{SER-FIQ\cite{SERFIQ}} &12.9329 &25.7687 &7.4298 &8.2722 &3.4282 &4.8334 &0.7354 &0.9747 &20.9112 &21.7010 &20.1676 &31.2656 &118.0898 &143.3904 &6.3280 &9.6586 & \text{1.119} \\
                                                                                    &                                                   & \multicolumn{1}{|l|}{FaceQnet\cite{hernandez2019faceqnet,faceqnetv1}} &16.8061 &32.0431 &8.6961 &9.5287 &8.2606 &9.7281 &0.8897 &1.0593 &22.5921 &23.3901 &44.2566 &112.0901 &171.8399 &195.8061 &8.2501 &12.3822 & \text{1.730} \\
                                                                                    &                                                   & \multicolumn{1}{|l|}{MagFace\cite{MagFace}} &12.3547 &23.5913 &6.9542 &7.3554 &4.7674 &6.5524 &0.5640 &0.8408 &20.5457 &20.9843 &26.4685 &37.6205 &158.0920 &170.8094 &6.9066 &10.4967 & \text{1.191} \\
                                                                                    &                                                   & \multicolumn{1}{|l|}{SDD-FIQA\cite{SDDFIQA}} &13.2530 &26.4840 &8.7657 &9.3844 &6.1389 &7.5409 &0.6815 &0.9629 &21.4104 &22.0132 &28.2482 &41.1283 &157.9446 &185.3076 &6.9932 &10.5790 & \text{1.333} \\
                                                                                    &                                                   & \multicolumn{1}{|l|}{CR-FIQA(L)\cite{boutros_2023_crfiqa}} &11.7700 &22.9612 &7.3672 &7.8950 &3.2877 &4.7947 &0.6925 &1.0117 &20.1932 &20.7468 &19.2650 &29.2828 &124.8701 &145.2975 &6.3553 &9.7637 & \text{1.087} \\
                                                                                    &                                                   & \multicolumn{1}{|l|}{DifFIQA(R)\cite{10449044}} &12.5725 &25.3111 &8.5525 &9.1986 &3.4596 &4.9557 &0.6848 &0.8695 &20.9897 &21.6053 &18.7803 &28.7571 &127.9433 &149.5653 &6.2401 &9.5956 & \text{1.118} \\
                                                                                    &                                                   & \multicolumn{1}{|l|}{eDifFIQA(L)\cite{babnikTBIOM2024}} &11.1933 &23.1214 &6.5873 &7.0306 &3.0401 &4.5657 &0.6807 &0.8480 &20.2461 &20.8277 &18.7740 &28.7591 &135.8406 &161.2112 &6.1970 &9.5114 & \text{\textbf{1.054}} \\
                                                                                    &                                                   & \multicolumn{1}{|l|}{GraFIQs(L)\cite{grafiqs}} &11.3480 &22.7072 &7.6775 &8.5207 &3.7566 &5.1929 &0.7242 &1.0396 &20.7962 &21.4087 &21.0312 &43.5089 &146.0971 &177.2098 &6.5364 &10.0445 & \text{1.184} \\
                                                                                    &                                                   & \multicolumn{1}{|l|}{CLIB-FIQA\cite{Ou_2024_CVPR}} &11.8083 &24.7242 &7.1435 &7.5612 &3.4113 &5.0553 &0.6740 &0.8424 &20.1961 &20.7819 &19.2309 &29.1490 &129.0720 &159.7752 &6.3966 &9.7010 & \text{1.086} \\
                                                                                    &                                                   & \multicolumn{1}{|l|}{ViT-FIQA(T)\cite{atzori2025vitfiqaassessingfaceimage}} &11.2283 &22.5349 &7.6068 &8.1232 &3.2002 &4.8028 &0.6535 &0.8959 &20.7642 &21.5687 &19.4687 &29.2202 &135.1589 &172.1503 &6.3344 &9.7641 & \text{1.093} \\ \cline{3-20}
                                                                                    &                                                   & \multicolumn{1}{|l|}{\ourmethod(S)} &11.6787 &23.5462 &7.8794 &8.4337 &3.5031 &4.8915 &0.4708 &0.8217 &21.4418 &21.9482 &20.9729 &31.9218 &135.7698 &159.0455 &6.8754 &10.4452 & \text{1.100} \\
                                                                                    &                                                   & \multicolumn{1}{|l|}{\ourmethod(L)} &11.1126 &21.7069 &7.0520 &7.5924 &3.2395 &4.6019 &0.6642 &0.9243 &21.0112 &21.6587 &19.4305 &41.9585 &120.9900 &139.1890 &6.3347 &9.7956 & \text{1.088} \\
     \hline \hline
     \multirow{17}{*}{\rotatebox[origin=c]{90}{MagFace\cite{MagFace}}}      & \multirow{3}{*}{\rotatebox[origin=c]{90}{IQA}}    & \multicolumn{1}{|l|}{BRISQUE\cite{BRISQE_IQA}} &17.6864 &40.0214 &12.0693 &24.8289 &13.9873 &24.4571 &1.1289 &1.5102 &23.8383 &25.2040 &89.6010 &201.1140 &192.3250 &207.3702 &10.2837 &16.0212                                                & \text{2.098} \\
                                                                                    &                                                   & \multicolumn{1}{|l|}{RankIQA\cite{liu2017rankiqa}} &14.9786 &33.8819 &10.9311 &25.7044 &14.0267 &24.4645 &1.1181 &1.5370 &25.1165 &26.2428 &63.3017 &204.2609 &194.5007 &209.4855 &10.1104 &15.8294                                                & \text{1.997} \\
                                                                                    &                                                   & \multicolumn{1}{|l|}{DeepIQA\cite{DEEPIQ_IQA}} &17.7900 &40.0657 &12.1006 &24.5400 &13.1472 &23.5528 &1.1600 &1.6144 &24.7203 &25.8737 &90.9446 &203.1154 &189.0341 &204.6778 &10.4174 &16.0207                                                & \text{2.099} \\
                                                                                    \cline{2-20}
                                                                                    & \multirow{14}{*}{\rotatebox[origin=c]{90}{FIQA}}   & \multicolumn{1}{|l|}{RankIQ\cite{RANKIQ_FIQA}} &14.5745 &35.3153 &11.8940 &23.6708 &11.1641 &20.6461 &0.8117 &1.2068 &22.2224 &23.3407 &37.7717 &120.1296 &162.0601 &178.7525 &9.2168 &13.8719 & \text{1.621} \\
                                                                                    &                                                   & \multicolumn{1}{|l|}{PFE\cite{PFE_FIQA}} &10.9958 &26.8477 &8.5978 &18.7198 &7.2913 &9.9045 &0.8036 &0.9456 &22.1603 &22.7276 &27.1296 &121.1811 &158.3491 &178.0278 &8.4615 &12.4810 & \text{1.299} \\
                                                                                    &                                                   & \multicolumn{1}{|l|}{SER-FIQ\cite{SERFIQ}} &12.1035 &27.3982 &8.6959 &18.4777 &4.9180 &10.3506 &0.8766 &1.3435 &21.6498 &22.1447 &23.5790 &57.9246 &144.1821 &164.7769 &7.6596 &11.3009 & \text{1.199} \\
                                                                                    &                                                   & \multicolumn{1}{|l|}{FaceQnet\cite{hernandez2019faceqnet,faceqnetv1}} &15.6008 &34.0705 &9.5463 &17.7526 &11.0813 &18.2528 &1.0152 &1.1748 &22.9953 &23.5379 &75.4822 &189.0045 &190.3212 &201.4877 &9.6437 &14.3579 & \text{1.794} \\
                                                                                    &                                                   & \multicolumn{1}{|l|}{MagFace\cite{MagFace}} &11.2581 &25.8970 &7.5040 &14.7059 &6.2644 &10.1572 &0.7061 &0.8648 &21.0710 &21.7435 &29.6591 &62.5618 &176.2881 &190.8106 &8.2233 &12.2409 & \text{1.176} \\
                                                                                    &                                                   & \multicolumn{1}{|l|}{SDD-FIQA\cite{SDDFIQA}} &12.1307 &29.6442 &9.5418 &14.1409 &9.5315 &14.1416 &0.8258 &0.9868 &22.1486 &22.9653 &31.4450 &91.4908 &180.6231 &196.5300 &8.4144 &12.4684 & \text{1.366} \\
                                                                                    &                                                   & \multicolumn{1}{|l|}{CR-FIQA(L)\cite{boutros_2023_crfiqa}} &11.3206 &23.4602 &8.1395 &13.9425 &4.9929 &6.3465 &0.8183 &0.9607 &21.0291 &21.6504 &22.2420 &48.2321 &151.1800 &177.1831 &7.7592 &11.4177 & \text{\textbf{1.072}} \\
                                                                                    &                                                   & \multicolumn{1}{|l|}{DifFIQA(R)\cite{10449044}} &11.4815 &28.0995 &9.8176 &19.8918 &5.4472 &11.8393 &0.8145 &0.9903 &21.6867 &22.7119 &22.3511 &63.1215 &151.1309 &176.8266 &7.6029 &11.1622 & \text{1.212} \\
                                                                                    &                                                   & \multicolumn{1}{|l|}{eDifFIQA(L)\cite{babnikTBIOM2024}} &10.6142 &25.6009 &7.6993 &12.5834 &4.7557 &10.9863 &0.8103 &1.0032 &21.0854 &21.5357 &22.1859 &62.8192 &161.2137 &176.4403 &7.5535 &11.0755 & \text{1.127} \\
                                                                                    &                                                   & \multicolumn{1}{|l|}{GraFIQs(L)\cite{grafiqs}} &10.9854 &24.2881 &8.0414 &14.6700 &5.4537 &12.3808 &0.9206 &1.2800 &21.2672 &21.8732 &24.7447 &72.9392 &160.8520 &183.0776 &8.0243 &11.7850 & \text{1.227} \\
                                                                                    &                                                   & \multicolumn{1}{|l|}{CLIB-FIQA\cite{Ou_2024_CVPR}} &11.3012 &27.3331 &8.1278 &13.4630 &5.3174 &11.6884 &0.7986 &0.9927 &20.9668 &21.4732 &22.6654 &63.0537 &149.8783 &180.2327 &7.7146 &11.2779 & \text{1.156} \\
                                                                                    &                                                   & \multicolumn{1}{|l|}{ViT-FIQA(T)\cite{atzori2025vitfiqaassessingfaceimage}} &10.1836 &25.5068 &8.6440 &15.3334 &4.9257 &7.4628 &0.7789 &0.9384 &21.6109 &22.2706 &22.2282 &48.2996 &151.1084 &176.6174 &7.6865 &11.3785 & \text{1.088} \\  \cline{3-20}
                                                                                    &                                                   & \multicolumn{1}{|l|}{\ourmethod(S)} &10.8409 &26.0254 &8.6941 &18.9216 &5.2360 &8.0943 &0.5961 &0.7915 &21.9904 &22.3856 &24.2468 &58.1889 &163.9470 &193.0470 &8.2556 &12.0825 & \text{1.132} \\
                                                                                    &                                                   & \multicolumn{1}{|l|}{\ourmethod(L)} &10.6828 &23.4924 &8.0307 &15.7318 &4.7140 &9.8454 &0.7896 &0.9318 &21.7903 &22.5506 &23.0156 &70.8622 &154.8593 &168.2000 &7.7346 &11.4954 & \text{1.136} \\
     \hline \hline
     \multirow{17}{*}{\rotatebox[origin=c]{90}{CurricularFace\cite{curricularFace}}}      & \multirow{3}{*}{\rotatebox[origin=c]{90}{IQA}}    & \multicolumn{1}{|l|}{BRISQUE\cite{BRISQE_IQA}} &15.1305 &33.9150 &12.0922 &14.8703 &11.1308 &12.6739 &1.0438 &1.2273 &23.4578 &25.2123 &50.4952 &141.1246 &162.8115 &173.6832 &8.2785 &12.4015                                                & \text{1.919} \\
                                                                                    &                                                   & \multicolumn{1}{|l|}{RankIQA\cite{liu2017rankiqa}} &12.7961 &29.8147 &10.9918 &14.8060 &11.1852 &14.7008 &1.0739 &1.2560 &24.3938 &26.4996 &50.7672 &144.9530 &164.5610 &175.5549 &8.1711 &12.2257                                                & \text{1.925} \\
                                                                                    &                                                   & \multicolumn{1}{|l|}{DeepIQA\cite{DEEPIQ_IQA}} &15.1597 &34.0296 &11.5896 &13.9643 &10.2402 &13.6260 &1.1158 &1.2964 &24.2026 &25.1478 &50.5253 &144.4031 &156.9123 &169.8764 &8.3842 &12.4514                                                & \text{1.922} \\
                                                                                    \cline{2-20}
                                                                                    & \multirow{14}{*}{\rotatebox[origin=c]{90}{FIQA}}   & \multicolumn{1}{|l|}{RankIQ\cite{RANKIQ_FIQA}} &12.5206 &28.9732 &11.4410 &13.8973 &9.0876 &12.3471 &0.7481 &0.9286 &21.5436 &23.7314 &31.1507 &44.6895 &132.0292 &152.0204 &7.6540 &11.2293 & \text{1.455} \\
                                                                                    &                                                   & \multicolumn{1}{|l|}{PFE\cite{PFE_FIQA}} &9.5227 &22.0627 &8.4776 &10.4507 &6.4967 &8.7405 &0.7953 &0.9206 &21.7226 &23.1964 &22.5460 &79.3338 &120.1340 &137.7429 &7.0903 &10.2311 & \text{1.296} \\
                                                                                    &                                                   & \multicolumn{1}{|l|}{SER-FIQ\cite{SERFIQ}} &10.4045 &23.2638 &7.7957 &9.4513 &3.8105 &5.8082 &0.8508 &0.9747 &21.1931 &22.6532 &18.4467 &32.7542 &117.7541 &138.2310 &6.2530 &9.2097 & \text{1.088} \\
                                                                                    &                                                   & \multicolumn{1}{|l|}{FaceQnet\cite{hernandez2019faceqnet,faceqnetv1}} &13.6078 &29.7680 &9.6275 &11.6946 &8.5796 &11.4121 &1.0069 &1.1323 &22.7618 &24.3032 &42.4117 &125.9255 &159.2218 &175.3376 &8.0810 &11.6636 & \text{1.707} \\
                                                                                    &                                                   & \multicolumn{1}{|l|}{MagFace\cite{MagFace}} &10.1786 &22.2762 &7.6523 &9.4275 &5.3519 &7.4888 &0.6803 &0.8408 &20.7965 &21.9147 &24.4510 &38.7748 &150.7265 &163.2633 &6.7940 &9.9874 & \text{1.173} \\
                                                                                    &                                                   & \multicolumn{1}{|l|}{SDD-FIQA\cite{SDDFIQA}} &10.5219 &24.3337 &9.4920 &11.5488 &8.3939 &11.1668 &0.7999 &0.9629 &21.7496 &23.4127 &26.2712 &44.8727 &142.4923 &162.1928 &6.9040 &10.0534 & \text{1.353} \\
                                                                                    &                                                   & \multicolumn{1}{|l|}{CR-FIQA(L)\cite{boutros_2023_crfiqa}} &10.1114 &21.0583 &7.5560 &9.5109 &4.0536 &5.9641 &0.8305 &1.0117 &20.7014 &21.3972 &17.3640 &29.9613 &119.2319 &149.5570 &6.3053 &9.2468 & \text{1.076} \\
                                                                                    &                                                   & \multicolumn{1}{|l|}{DifFIQA(R)\cite{10449044}} &9.8056 &23.1090 &9.9488 &11.7492 &3.7724 &5.9817 &0.7889 &0.9299 &21.0650 &22.7619 &17.0462 &29.5376 &124.2980 &141.5128 &6.1846 &9.1629 & \text{1.101} \\
                                                                                    &                                                   & \multicolumn{1}{|l|}{eDifFIQA(L)\cite{babnikTBIOM2024}} &8.9957 &20.3095 &7.5763 &8.9476 &3.5200 &5.6930 &0.7847 &0.9082 &20.5967 &21.9945 &16.9471 &29.4616 &131.5937 &148.3697 &6.1640 &9.0445 & \text{\textbf{1.039}} \\
                                                                                    &                                                   & \multicolumn{1}{|l|}{GraFIQs(L)\cite{grafiqs}} &9.6944 &20.7428 &7.4491 &9.1992 &4.0808 &6.1433 &0.8795 &1.0396 &20.8857 &22.5015 &19.5003 &47.3121 &125.1931 &141.4252 &6.4942 &9.4559 & \text{1.130} \\
                                                                                    &                                                   & \multicolumn{1}{|l|}{CLIB-FIQA\cite{Ou_2024_CVPR}} &9.7681 &21.7309 &8.1028 &9.6342 &3.8405 &6.0757 &0.7903 &0.9149 &20.4888 &21.8968 &17.3670 &29.9729 &123.1865 &141.8351 &6.3213 &9.2505 & \text{1.067} \\
                                                                                    &                                                   & \multicolumn{1}{|l|}{ViT-FIQA(T)\cite{atzori2025vitfiqaassessingfaceimage}} &8.8986 &20.8904 &8.6058 &10.5926 &3.9729 &5.8003 &0.7705 &0.8959 &21.4389 &22.7171 &17.3040 &29.5901 &124.9113 &144.6134 &6.3369 &9.3218 & \text{1.072} \\ \cline{3-20}
                                                                                    &                                                   & \multicolumn{1}{|l|}{\ourmethod(S)} &9.4222 &21.7211 &8.5354 &10.4881 &4.3035 &6.2893 &0.6411 &0.8217 &21.6933 &23.2785 &19.0764 &33.4702 &133.9656 &152.6510 &6.8556 &9.9116 & \text{1.105} \\
                                                                                    &                                                   & \multicolumn{1}{|l|}{\ourmethod(L)} &9.4004 &20.7287 &7.9431 &9.4951 &3.7846 &5.4664 &0.7813 &0.9243 &21.4933 &22.2463 &17.6858 &45.2779 &122.3199 &149.8777 &6.2979 &9.3203 & \text{1.092} \\
     \hline \hline
     \multirow{17}{*}{\rotatebox[origin=c]{90}{\text{Average}}}      & \multirow{3}{*}{\rotatebox[origin=c]{90}{\text{IQA}}}    & \multicolumn{1}{|l|}{\text{BRISQUE\cite{BRISQE_IQA}}} & \text{17.2921} & \text{37.4839} & \text{11.6161} & \text{17.2688} & \text{11.5877} & \text{16.5614} & \text{1.0216} & \text{1.2980} & \text{23.6475} & \text{25.3678} & \text{63.4290} & \text{136.6868} & \text{176.6849} & \text{195.2231} & \text{8.9241} & \text{13.8416} & \text{1.944} \\
     &                                                   & \multicolumn{1}{|l|}{\text{RankIQA\cite{liu2017rankiqa}}} & \text{14.6562} & \text{32.8486} & \text{10.6554} & \text{17.2882} & \text{11.4028} & \text{16.8742} & \text{1.0417} & \text{1.3262} & \text{24.8283} & \text{26.5830} & \text{57.1295} & \text{140.1877} & \text{178.6741} & \text{197.4032} & \text{8.8166} & \text{13.6543} & \text{1.904} \\
     &                                                   & \multicolumn{1}{|l|}{\text{DeepIQA\cite{DEEPIQ_IQA}}} & \text{17.4275} & \text{37.8800} & \text{11.4080} & \text{16.7940} & \text{10.6134} & \text{15.6658} & \text{1.0842} & \text{1.3759} & \text{24.4935} & \text{25.9072} & \text{63.7045} & \text{139.1889} & \text{172.4577} & \text{192.2771} & \text{9.0438} & \text{13.8877} & \text{1.935} \\ \cline{2-20}
     & \multirow{14}{*}{\rotatebox[origin=c]{90}{\text{FIQA}}}   & \multicolumn{1}{|l|}{\text{RankIQ\cite{RANKIQ_FIQA}}} & \text{14.4688} & \text{33.0985} & \text{11.1557} & \text{16.5629} & \text{9.0728} & \text{14.0290} & \text{0.7214} & \text{0.9982} & \text{21.9393} & \text{23.6720} & \text{34.1121} & \text{64.2925} & \text{144.2126} & \text{164.1697} & \text{8.1264} & \text{12.2979} & \text{1.501} \\
     &                                                   & \multicolumn{1}{|l|}{\text{PFE\cite{PFE_FIQA}}} & \text{10.7654} & \text{24.8835} & \text{8.1809} & \text{12.4582} & \text{6.2738} & \text{8.6987} & \text{0.7681} & \text{0.9269} & \text{21.9054} & \text{23.0988} & \text{25.0038} & \text{78.3814} & \text{138.5025} & \text{161.9303} & \text{7.5209} & \text{11.1023} & \text{1.274} \\
     &                                                   & \multicolumn{1}{|l|}{\text{SER-FIQ\cite{SERFIQ}}} & \text{11.7671} & \text{25.9661} & \text{7.9244} & \text{12.1211} & \text{3.9885} & \text{6.8244} & \text{0.8158} & \text{1.0669} & \text{21.4517} & \text{22.6753} & \text{20.9409} & \text{39.2575} & \text{\textbf{128.0986}} & \text{\textbf{150.8112}} & \text{6.6921} & \text{10.0655} & \text{1.127} \\
     &                                                   & \multicolumn{1}{|l|}{\text{FaceQnet\cite{hernandez2019faceqnet,faceqnetv1}}} & \text{15.3219} & \text{32.8376} & \text{9.1686} & \text{12.9199} & \text{9.2325} & \text{12.7158} & \text{0.9797} & \text{1.1247} & \text{22.9221} & \text{24.2385} & \text{53.2579} & \text{123.0745} & \text{176.1318} & \text{193.7110} & \text{8.6191} & \text{12.7754} & \text{1.697} \\
     &                                                   & \multicolumn{1}{|l|}{\text{MagFace\cite{MagFace}}} & \text{11.2363} & \text{24.8217} & \text{7.3846} & \text{10.5763} & \text{5.3338} & \text{7.9233} & \text{0.6577} & \text{0.8468} & \text{20.8699} & \text{21.8679} & \text{27.0610} & \text{44.7364} & \text{161.4849} & \text{178.8523} & \text{7.2695} & \text{10.9020} & \text{1.170} \\
     &                                                   & \multicolumn{1}{|l|}{\text{SDD-FIQA\cite{SDDFIQA}}} & \text{11.9374} & \text{27.5554} & \text{9.1053} & \text{11.3156} & \text{7.9687} & \text{11.3369} & \text{0.7768} & \text{0.9689} & \text{21.9157} & \text{23.1591} & \text{29.2775} & \text{54.8726} & \text{160.0526} & \text{180.7451} & \text{7.3869} & \text{11.0349} & \text{1.339} \\
     &                                                   & \multicolumn{1}{|l|}{\text{CR-FIQA(L)\cite{boutros_2023_crfiqa}}} & \text{11.0256} & \text{22.4580} & \text{7.6669} & \text{10.3329} & \text{3.9986} & \text{\textbf{5.8178}} & \text{0.7878} & \text{0.9989} & \text{20.7151} & \text{\textbf{21.4380}} & \text{19.8112} & \text{35.1694} & \text{133.8246} & \text{157.7913} & \text{6.7498} & \text{10.1356} & \text{1.069} \\
     &                                                   & \multicolumn{1}{|l|}{\text{DifFIQA(R)\cite{10449044}}} & \text{11.2715} & \text{26.4102} & \text{9.3969} & \text{13.6421} & \text{4.1526} & \text{7.3710} & \text{0.7693} & \text{0.9299} & \text{21.3856} & \text{22.8614} & \text{19.5907} & \text{38.7417} & \text{135.2986} & \text{156.6299} & \text{6.6275} & \text{9.9482} & \text{1.141} \\
     &                                                   & \multicolumn{1}{|l|}{\text{eDifFIQA(L)\cite{babnikTBIOM2024}}} & \text{10.2533} & \text{23.6296} & \text{\textbf{7.1858}} & \text{\textbf{9.3600}} & \text{\textbf{3.7155}} & \text{6.8231} & \text{0.7651} & \text{0.9169} & \text{20.7351} & \text{21.9559} & \text{\textbf{19.4982}} & \text{38.5968} & \text{142.7411} & \text{163.2074} & \text{\textbf{6.5960}} & \text{\textbf{9.8555}} & \text{\textbf{1.065}} \\
     &                                                   & \multicolumn{1}{|l|}{\text{GraFIQs(L)\cite{grafiqs}}} & \text{10.6423} & \text{22.8738} & \text{7.7212} & \text{10.8559} & \text{4.4097} & \text{7.7050} & \text{0.8411} & \text{1.0997} & \text{21.0936} & \text{22.4208} & \text{21.9428} & \text{50.3573} & \text{144.1127} & \text{165.0988} & \text{6.9794} & \text{10.3948} & \text{1.163} \\
     &                                                   & \multicolumn{1}{|l|}{\text{CLIB-FIQA\cite{Ou_2024_CVPR}}} & \text{10.9522} & \text{25.2768} & \text{7.6903} & \text{10.0235} & \text{4.1597} & \text{7.3971} & \text{0.7633} & \text{0.9162} & \text{\textbf{20.6788}} & \text{21.8731} & \text{19.9236} & \text{38.9506} & \text{134.8840} & \text{158.1936} & \text{6.7570} & \text{10.0465} & \text{1.094} \\
     &                                                   & \multicolumn{1}{|l|}{\text{ViT-FIQA(T)\cite{atzori2025vitfiqaassessingfaceimage}}} & \text{\textbf{10.0647}} & \text{23.6491} & \text{8.2726} & \text{11.1958} & \text{3.9168} & \text{5.9321} & \text{0.7434} & \text{0.9065} & \text{21.3962} & \text{22.5425} & \text{19.8830} & \text{\textbf{35.1244}} & \text{137.9110} & \text{162.4139} & \text{6.7301} & \text{10.1457} & \text{1.075} \\ \cline{3-20}
     &                                                   & \multicolumn{1}{|l|}{\text{\ourmethod(S)}} & \text{10.6131} & \text{24.4598} & \text{8.3627} & \text{12.5906} & \text{4.1878} & \text{6.2847} & \text{\textbf{0.5740}} & \text{\textbf{0.8142}} & \text{21.8141} & \text{23.0343} & \text{21.5416} & \text{39.6972} & \text{145.3297} & \text{166.4868} & \text{7.2623} & \text{10.7849} & \text{1.104} \\
     &                                                   & \multicolumn{1}{|l|}{\text{\ourmethod(L)}} & \text{10.3821} & \text{\textbf{22.2977}} & \text{7.6498} & \text{10.8506} & \text{3.8887} & \text{6.4677} & \text{0.7541} & \text{0.9262} & \text{21.5529} & \text{22.3777} & \text{20.5959} & \text{48.3652} & \text{132.2035} & \text{154.9409} & \text{6.7339} & \text{10.1699} & \text{1.094} \\
     \hline
    \end{tabular}}
    \end{center}
    \vspace{-3mm}
    \end{table*}

\section{Methodology}
\label{sec:methodology}

This section presents our proposed \ourmethod that leverages quality training targets accumulated over multiple training epochs to achieve stable FIQA. Unlike methods that rely on single-epoch measurements \cite{MagFace, PFE_FIQA, boutros_2023_crfiqa}, our approach aggregates relative point margin values across the training trajectory, providing temporally-stable quality estimates. \vspace{-2mm}

\subsection{Limitations of Single-Epoch Quality Estimation}
\label{subsec:motivation}

As discussed in Sec. \ref{subsec:related_fiqa}, existing FR-integrated FIQA methods estimate quality at specific training iterations or final epochs. For instance, CR-FIQA \cite{boutros_2023_crfiqa} computes quality as \( q(x) = \text{CCS}(x)/\text{NNCCS}(x) \) at each epoch, where CCS and NNCCS represent the sample's similarity to its own class center versus competing class centers. However, during training, each epoch generates a different induced class clustering as the feature space evolves. This makes \( q_j(x) \) vary across epochs \( j \), creating a moving target that leads to: (1) fluctuating training objectives that hamper optimization stability \cite{DBLP:conf/nips/GrillSATRBDPGAP20, DBLP:conf/icml/SutskeverMDH13}, (2) unstable sample quality rankings across training iterations. Our cumulative approach addresses these issues by aggregating quality measurements across the entire training trajectory, reducing dependency on any single epoch's clustering structure. \vspace{-2mm}

\subsection{Cumulative Average of Relative Point Margin (CARPM)}
\label{subsec:carpm} A cumulative average (\( CA \)), is a running average of a sequence of values that is updated incrementally as new data points are added \cite{ya1963statistical}. At each step, the cumulative average is computed as the average of all previous values up to that point. Mathematically, for a sequence of values \( x_1, x_2, \dots, x_n \), the cumulative average at step \( n \) is given by:
\begin{equation}
    CA_n = \frac{1}{n} \sum_{i=1}^{n} x_i = \frac{x_{n}+(n-1) CA_{n-1}}{n}
    \label{eq:cumulative_average}.
\end{equation}
In the context of training a FR model, we may view the model's training process as an induced dynamic class clustering. Since each iteration generates a different induced class clustering due to the constantly evolving nature of the learned feature space, the target values, such as the relative point margin \cite{DBLP:conf/nips/Ben-DavidA08}, are subject to change. This results in fluctuating rankings of image qualities, making it difficult to establish stable target values for optimization, illustrated in Fig. \ref{fig:high_quality_id6987}. To mitigate this issue and stabilize the training, we propose using the cumulative average of the relative point margin (CARPM) as the target during training. By averaging the target values over multiple epochs, we smooth out the fluctuations caused by the dynamic nature of the induced class clustering. For each sample \( x \) at epoch \( n \), we calculate its cumulative average quality score \( Q(x)_n \) as follows:
\begin{equation}
    Q(x)_n = \frac{1}{n} \sum_{j=1}^{n} q_j(x)
    \label{eq:carpm},
\end{equation}
where \( q_j(x) \) is the quality score for sample \( x \) at epoch \( j \). Given the dynamic nature of induced feature clustering during training, we can model the quality score at epoch $n$ as:
\begin{equation}
    q_n(x) = Q^*(x) + \varepsilon_n
    \label{eq:quality_model},
\end{equation}
where $Q^*(x)$ represents the "true" quality value and $\varepsilon_n$ is the noise from induced clustering variations at epoch $n$ and sample $x$. We can conceptually characterize this noise as if we were to train our model $m$ times with different initializations, which would yield $m$ different quality scores $q_n^1(x), q_n^2(x), ..., q_n^m(x)$ for the sample $x$ at epoch $n$. Such variations would allow to empirically estimate $\varepsilon_n$ distribution.

\textbf{Theoretical Analysis of CARPM:} We now provide a theoretical analysis of the cumulative average of relative point margin (CARPM), demonstrating its statistical advantages over single-epoch quality estimates. We begin with a formal characterization of the quality estimation problem.

\begin{assumption}[Quality Score Noise Model]
\label{assump:noise_model}
Assuming both $q_n(x)$ and $Q^*(x)$ are appropriately normalized, we model the quality score at epoch $n$ as $q_n(x) = Q^*(x) + \varepsilon_n$, where $\varepsilon_n \sim \mathcal{N}(0, \sigma_n^2)$ is a normally distributed error term with zero mean.
\end{assumption}

This noise model captures the stochastic nature of the training process and the resulting variability in quality estimates. The zero-mean property implies that the quality score $q_n(x)$ is an unbiased estimate of the true quality $Q^*(x)$, i.e., $\EX[q_n(x)] = \EX[Q^*(x)]$. This is reasonable because during training, the network optimizes to minimize classification loss, tending to produce embeddings that correctly represent the identities. While estimations fluctuate, these fluctuations are equally likely to over or underestimate the true quality rather than showing systematic bias in either direction.

\begin{assumption}[Decaying Epoch-to-Epoch Correlation]
\label{assump:correlated_noise}
The noise variance is stationary across epochs ($\sigma_n^2 = \sigma^2$ for all $n$), and the correlation between $\varepsilon_i$ and $\varepsilon_j$ decays with epoch distance following an AR(1)-type structure: $Cov(\varepsilon_i, \varepsilon_j) = \rho^{|i-j|}\sigma^2$ for some $\rho \in [0,1)$. Uncorrelated errors are the special case $\rho = 0$.
\end{assumption}

We deliberately do not assume the errors $\varepsilon_i$ and $\varepsilon_j$ to be independent across epochs, which would be debatable due to sequential parameter updates; we only require their correlation to fade with epoch distance. This weaker requirement is supported by: (1) stochasticity from mini-batch selection and data augmentation; (2) a direct empirical lag-correlation analysis, discussed after Property \ref{prop:variance} and empirical validation; and (3) the fact that, as shown below, the statistical benefits of averaging survive under it, with constants inflated by $(1+\rho)/(1-\rho)$ relative to the uncorrelated case. 

\begin{property}[Unbiasedness of CARPM]
\label{prop:unbiased}
Under Assumption \ref{assump:noise_model}, the cumulative average quality score $Q(x)_n$ is an unbiased estimator of the true quality $Q^*(x)$.
\end{property}

\begin{proof}
The cumulative average quality score after $n$ epochs is:
\begin{footnotesize}
\begin{equation}
Q(x)_n = \frac{1}{n} \sum_{j=1}^{n} q_j(x) = \frac{1}{n} \sum_{j=1}^{n} (Q^*(x) + \varepsilon_j) = Q^*(x) + \frac{1}{n} \sum_{j=1}^{n} \varepsilon_j.\end{equation} 
\end{footnotesize}
Assuming the noise terms $\varepsilon_j$ have zero mean and finite variance $\sigma_j^2$, the expected value of this estimate becomes:
\begin{footnotesize}
    \begin{align}
        \EX[Q(x)_n]   &= \EX\left[Q^*(x) + \frac{1}{n} \sum_{j=1}^{n} \varepsilon_j\right]  =\EX[Q^*(x)] + \frac{1}{n} \sum_{j=1}^{n} \EX[\varepsilon_j] \\
                    &= Q^*(x) + 0\
                    = Q^*(x).
\end{align}
\end{footnotesize}

Hence, our proposed metric $Q(x)_n$ is unbiased if $q_n(x)$ is an unbiased estimate of the true quality $Q^*(x)$.
\end{proof}

\begin{property}[Reduced Variance Property]
\label{prop:variance}
Under Assumptions \ref{assump:noise_model} and \ref{assump:correlated_noise}, the variance of $Q(x)_n$ decreases at a rate proportional to $1/n$ as the number of epochs increases, with
\begin{footnotesize}
\begin{align}
\label{eq:corr_variance_exact}
Var(Q(x)_n) &= \frac{1}{n^2}\left[n\sigma^2 + 2\sigma^2\sum_{k=1}^{n-1}(n-k)\rho^k\right] \\&\xrightarrow{n\to\infty} \frac{\sigma^2(1+\rho)}{1-\rho}\cdot\frac{1}{n}.
\end{align}
\end{footnotesize}
\end{property}
\begin{proof}
As $Var(aX+b)=a^2Var(X)$, the variance of the cumulative average is
\begin{footnotesize}
\begin{align}
        Var(Q(x)_n) &= Var(Q^*(x))+Var\left(\frac{1}{n} \sum_{j=1}^{n} \varepsilon_j\right) \\
        &= 0 + \frac{1}{n^2}\sum_{i,j=1}^n Cov(\varepsilon_i,\varepsilon_j)
                    = \frac{1}{n^2}\sum_{i,j=1}^n \rho^{|i-j|}\sigma^2.
\end{align}\end{footnotesize}
Grouping terms by lag $k=|i-j|$, there are $n$ terms at lag $0$ and $2(n-k)$ index pairs at each lag $k=1,\dots,n-1$, giving Eq. \ref{eq:corr_variance_exact}'s finite-$n$ expression directly. As $n\to\infty$, $\frac{1}{n}\sum_{k=1}^{n-1}(n-k)\rho^k \to \sum_{k=1}^{\infty}\rho^k = \frac{\rho}{1-\rho}$, so $Var(Q(x)_n) \sim \frac{\sigma^2}{n}\left[1 + \frac{2\rho}{1-\rho}\right] = \frac{\sigma^2}{n}\cdot\frac{1+\rho}{1-\rho}$. In the uncorrelated special case $\rho=0$, this reduces to the familiar $Var(Q(x)_n) = \sigma^2/n$.
This shows that as $n$ increases, the variance of our quality estimate decreases proportionally to $1/n$, providing a more stable target for optimization.
\end{proof}

The decay in Assumption \ref{assump:correlated_noise} is exactly what Property \ref{prop:variance} needs: the $O(1/n)$ variance reduction survives any correlation structure that \textit{decays} with epoch distance, merely inflated by a constant factor $(1+\rho)/(1-\rho) \geq 1$ relative to the uncorrelated case. This is not automatic, however: under an equicorrelated model with no decay ($Cov(\varepsilon_i,\varepsilon_j) = \rho\sigma^2$ for all $i \neq j$, i.e., a single shared source of noise rather than one that fades with temporal distance), the same derivation instead gives $Var(Q(x)_n) \to \rho\sigma^2$ as $n \to \infty$, a hard variance floor rather than continued reduction. The relevant empirical question is therefore not whether epoch-to-epoch correlation exists, sequential optimization makes some correlation plausible, as discussed, but whether it decays with epoch distance, as Assumption \ref{assump:correlated_noise} requires. We verify this directly: computing the lag-averaged empirical correlation of per-epoch relative point margin residuals, after removing both the epoch-wise shift/compression documented in Sec. \ref{sec:results}, across all eight reconstructable training configurations, we find that correlation decays approximately geometrically with epoch lag, with a fitted rate that depends on training scale: $\hat\rho \approx 0.77$ for the small protocol (34 epochs, inflation factor $(1{+}\hat\rho)/(1{-}\hat\rho) \approx 7.7$) and $\hat\rho \approx 0.42$ for the large protocol (18 epochs, inflation factor $\approx 2.5$), consistent across both ArcFace and CurricularFace loss and both cumulative and non-cumulative targets, training configurations described in Sec. \ref{sec:exp}. Both values fall into the decaying-correlation regime of Assumption \ref{assump:correlated_noise} rather than the equicorrelated worst case, supporting that cumulative averaging retains its $O(1/n)$ stabilizing effect under realistic, non-independent training dynamics, at a $2.5$-$7.7\times$ larger constant than the idealized uncorrelated case.

\begin{property}[Lower Mean Squared Error]
\label{prop:mse}
Under Assumptions \ref{assump:noise_model} and \ref{assump:correlated_noise}, the mean squared error of $Q(x)_n$ in approximating $Q^*(x)$ is lower than that of any single-epoch estimate $q_n(x)$ by a factor proportional to $1/n$.
\end{property}

\begin{proof}
For any estimator using a single epoch's measurement $q_n(x)$, the expected squared error is:
\begin{equation}\EX[(q_n(x) - Q^*(x))^2] = \EX[\varepsilon_n^2] = \sigma^2.\end{equation}

In contrast, since $Q(x)_n$ is an unbiased estimator of $Q^*(x)$ (Property \ref{prop:unbiased}), its expected squared error equals its variance, which Property \ref{prop:variance} bounds:
\begin{footnotesize}
\begin{align*}
     \EX[(Q(x)_n - Q^*(x))^2] &= Var(Q(x)_n) \sim \frac{1+\rho}{1-\rho}\cdot\frac{\sigma^2}{n}\\
                            &< \sigma^2 = \EX[(q_n(x) - Q^*(x))^2] \text{ for } n > \tfrac{1+\rho}{1-\rho},
\end{align*}
\end{footnotesize}
reducing to exactly $\sigma^2/n$ in the uncorrelated special case $\rho = 0$.
The cumulative average approach reduces estimation error by a factor proportional to $n$ compared to using single-epoch measurements.
\end{proof}

\begin{proposition}[Convergence of Ranking Probability]
\label{thm:ranking}
Under Assumptions \ref{assump:noise_model} and \ref{assump:correlated_noise}, for any two samples $x_a$ and $x_b$ with true qualities $Q^*(x_a) > Q^*(x_b)$, the probability of incorrect ranking using $Q(x)_n$ decreases at a rate of at least $1/n$ and approaches zero as $n \rightarrow \infty$.
\end{proposition}

\begin{proof}
Consider two samples $x_a$ and $x_b$ with true qualities $Q^*(x_a) > Q^*(x_b)$, and define $\delta = Q^*(x_a) - Q^*(x_b) > 0$ as their true quality difference. For a single epoch $j$, the probability of incorrect ranking is:
\begin{equation}
P(q_j(x_a) < q_j(x_b)) = P(\varepsilon_j^b - \varepsilon_j^a > \delta).\end{equation}
If we denote $Z_j = \varepsilon_j^b - \varepsilon_j^a$, then $\EX[Z_j] = 0$ by unbiased assumption and $Var(Z_j) := \sigma_Z^2$. Using Chebyshev's inequality \cite{AP91}:
\begin{equation}
P(Z_j > \delta) \leq P(|Z_j| > \delta) \leq \frac{\sigma_Z^2}{\delta^2}.\end{equation}
This establishes a lower bound on the ranking accuracy for a single epoch, which depends solely on the ratio of noise variance to the squared true quality difference. Whereas, the probability of incorrect ranking for cumulative averaging is:
\begin{equation}
P(Q(x_a)_n < Q(x_b)_n) = P\left(\frac{1}{n}\sum_{j=1}^{n}(\varepsilon_j^b - \varepsilon_j^a) > \delta\right).\end{equation}
Define $\bar{Z}_n = \frac{1}{n}\sum_{j=1}^{n}Z_j$, with $\EX[\bar{Z}_n] = 0$ from our unbiased assumption. Since the noise trajectories of the two distinct samples are drawn independently of each other, $Z_j$ inherits the epoch-correlation structure of Assumption \ref{assump:correlated_noise}: $Cov(Z_i, Z_j) = \rho^{|i-j|}\sigma_Z^2$, where $\sigma_Z^2$ is the variance of $Z_j$. The same lag-grouping argument as in the proof of Property \ref{prop:variance} therefore gives $Var(\bar{Z}_n) \leq \frac{1+\rho}{1-\rho}\cdot\frac{\sigma_Z^2}{n}$, and by Chebyshev's inequality:
\begin{equation}
P(|\bar{Z}_n - \EX[\bar{Z}_n]| \geq \delta) \leq \frac{Var(\bar{Z}_n)}{\delta^2} \leq \frac{1+\rho}{1-\rho}\cdot\frac{\sigma_Z^2}{n\delta^2}.\end{equation}
Since incorrect ranking requires $\bar{Z}_n > \delta$, we have:
\begin{align}
P(Q(x_a)_n < Q(x_b)_n) &= P(\bar{Z}_n > \delta) \leq P(|\bar{Z}_n| \geq \delta)\\
&\leq \frac{1+\rho}{1-\rho}\cdot\frac{\sigma_Z^2}{n\delta^2} \xrightarrow{n\to\infty} 0.
\end{align}
This establishes that the probability of incorrect ranking decreases at least as fast as $1/n$ and approaches zero as $n$ increases.
\end{proof}

By incorporating the cumulative average, we achieve \textbf{reduced variance in quality estimates} (Property \ref{prop:variance}), \textbf{lower mean squared error in approximating true quality} (Property \ref{prop:mse}), and \textbf{more consistent ranking of samples} (Proposition \ref{thm:ranking}). These properties ensure that target values remain more stable throughout training, reducing the impact of transient changes in the induced class clustering structure and enabling the model to learn a more consistent representation of FIQ.

\textbf{Quality Regression Model for Inference:} While the CARPM provides a robust quality measure, it is only available for samples in the training dataset where class centers are known. In practical applications, we need to assess the quality of previously unseen face images. To address this limitation, we develop a regression model that can predict the CARPM for any given face image. Our approach simultaneously trains a FR model and a quality regression branch. The FR model learns identity-discriminative features through ArcFace loss \cite{deng2019arcface}, while the quality regression branch learns to predict the cumulative average quality scores from the embeddings. During training, we track and accumulate the relative point margin values for each sample across epochs, providing increasingly stable quality targets as training progresses. Specifically, we extend an FR backbone network with a regression head that takes the face embeddings as input and outputs predicted quality scores. For each training sample $x_i$ at epoch $n$, we compute: (a) The current relative point margin $q_n(x_i)$, (b) the cumulative average quality score $Q(x_i)_n$ using all previous epochs, (c) the predicted quality score from the regression head $\hat{Q}(x_i)_n$. The model is trained with a combined loss function:
\begin{equation}
\label{eq:total_loss}
   \mathcal{L} = \mathcal{L}_{Arc} + \lambda \mathcal{L}_{Q},
\end{equation}
where $\mathcal{L}_{Arc}$ is the ArcFace loss \cite{deng2019arcface}, and $\mathcal{L}_{Q}$ is the quality prediction loss measuring the difference between the predicted scores and the cumulative average targets. We use Smooth L1 loss for $\mathcal{L}_{Q}$ due to its robustness to outliers and stable convergence properties by following \cite{boutros_2023_crfiqa}:
\begin{equation}
\label{eq:quality_loss}
\mathcal{L}_{Q} = \frac{1}{N}\sum\limits_{i=1}^{N} \ell(Q(x_i)_n, \hat{Q}(x_i)_n),
\end{equation}
where the loss function $\ell$ for each sample is defined as:
\begin{equation}
\ell(x, y) =
\begin{cases}
\frac{0.5 (x - y)^2}{\beta}, & \text{if } |x - y| < \beta \\
|x - y| - 0.5\beta, & \text{otherwise}
\end{cases},
\end{equation}
with $\beta = 0.5$ chosen as the threshold parameter that determines the transition between L1 and L2 behaviour, and $\lambda$ is a hyperparameter that balances the recognition and quality prediction objectives. We set $\lambda=10$ for all experiments based on preliminary studies showing this provides an appropriate balance between the two learning tasks \cite{boutros_2023_crfiqa}. This training paradigm enables our model to predict CARPM values for unseen face images. By learning from the aggregated quality measurements captured during FR training, the quality regression branch can effectively generalize to new samples. Algorithm \ref{alg:carpm_fiqa} summarizes the complete training procedure for \ourmethod, detailing how cumulative quality scores are computed and used to train the regression model alongside the FR objective.

\section{Experimental setup}
\label{sec:exp}

\textbf{Model training and implementation details:}
We evaluate our proposed method, \ourmethod, under two distinct protocols, small and large, based on the choice of training datasets and model architectures, following \cite{MagFace,boutros_2023_crfiqa,grafiqs}. Both protocols utilize common FR architectures, ResNet50 and ResNet100 \cite{DBLP:conf/cvpr/HeZRS16}, with modifications as specified in Sec. \ref{sec:methodology}. The small protocol employs ResNet50 trained on CASIA-WebFace \cite{casia_webface} (denoted as \ourmethod(S)), while the large protocol utilizes ResNet100 trained on MS1MV2 \cite{guo2016ms,deng2019arcface} (denoted as \ourmethod(L)). Both networks are initialized with Xavier initialization. The MS1MV2 dataset, a refined version of MS-Celeb-1M \cite{guo2016ms} by \cite{deng2019arcface}, consists of around 5.82 million images of 85,742 identities, whereas CASIA-WebFace contains 494,414 images of  10,575 identities \cite{casia_webface}. The training follows ArcFace settings \cite{deng2019arcface}, using a scale parameter $s$ of 64 and a margin $m$ of 0.5. Models are trained with Stochastic Gradient Descent (SGD) at an initial learning rate of 1e-1, a mini-batch size of 512, a momentum of 0.9, and a weight decay 5e-4. Data augmentation includes only random horizontal flipping with a probability of 0.5. For \ourmethod(S), the learning rate is reduced by a factor of 10 at 20K and 28K iterations, with training stopping at 32K iterations. For \ourmethod(L), the learning rate reduction occurs at 100K and 160K iterations, with training ending at 180K iterations. All images are aligned and cropped to $112 \times 112$ \cite{deng2019arcface} and normalized to pixel values between -1 and 1. All experiments are conducted using PyTorch 1.7.1 \cite{NEURIPS2019_9015} and trained on one Linux machine (Ubuntu 20.04.2 LTS) with 4 Nvidia HGX A100 GPUs with 40GB VRAM each, 256 CPU cores, and 1024GB of RAM. \textbf{Evaluation benchmarks and metrics:} To assess the generalizability of \ourmethod, we evaluate it on eight challenging benchmarks: Labeled Faces in the Wild (LFW) \cite{LFWTech}, AgeDB-30 \cite{agedb}, Celebrities in Frontal-Profile in the Wild (CFP-FP) \cite{cfp-fp}, Cross-Age LFW (CALFW) \cite{CALFW}, Adience \cite{Adience}, Cross-Pose LFW (CPLFW) \cite{CPLFWTech}, Cross-Quality LFW (XQLFW) \cite{XQLFW}, and IJB-C \cite{ijbc}. These benchmarks facilitate comparisons with SOTA FIQA methods \cite{boutros_2023_crfiqa,MagFace,DBLP:conf/iwbf/BabnikDS23,babnikTBIOM2024} and provide insights into the robustness of \ourmethod. Performance is measured using Error-versus-Discard Characteristic (EDC) curves \cite{GT07}, which assess the impact of discarding low-quality face images on face verification performance. The False Non-Match Rate (FNMR) is evaluated at fixed False Match Rate (FMR) thresholds \cite{iso_metric}, specifically at $1e-3$, recommended for border control by Frontex \cite{frontex2015best}, and $1e-4$. Additionally, the Area Under the Curve (AUC) and partial Area Under the Curve (pAUC) of EDC is reported to quantify verification performance across rejection rates. pAUC quantifies verification performance by considering only a specific portion of the EDC, up to a rejection rate of 30\% by following \cite{10449044, babnikTBIOM2024, Ou_2024_CVPR, DBLP:journals/tbbis/SchlettRTB24}. To examine the impact of FIQA across different FR models, we test \ourmethod on four CNN-based
FR systems (ArcFace \cite{deng2019arcface}, ElasticFace (ElasticFace-Arc) \cite{elasticface}, MagFace \cite{MagFace}, CurricularFace \cite{curricularFace}) as well as two ViT-based FR solutions (TransFace \cite{transface}, SwinFace \cite{swinface}). Each model processes $112 \times 112$ aligned images to generate 512-dimensional feature embeddings. Officially released models by each FR solution are used. Note that all evaluations are performed under cross-model settings, i.e., the models used to learn FIQA are different than the ones used to extract feature representation of face images, demonstrating the generalizability of our approach. \textbf{Comparisons with SOTA FIQ:} We compare \ourmethod against fifteen quality assessment methods. These include three general image quality assessment (IQA) techniques, BRISQUE \cite{BRISQE_IQA}, RankIQA \cite{liu2017rankiqa}, and DeepIQA \cite{DEEPIQ_IQA}, which have been shown to correlate with face utility \cite{BiyingWACV}. Additionally, we benchmark against twelve SOTA FIQA methods: RankIQ \cite{RANKIQ_FIQA}, PFE \cite{PFE_FIQA}, SER-FIQ \cite{SERFIQ}, FaceQnet (v1) \cite{hernandez2019faceqnet,faceqnetv1}, MagFace \cite{MagFace}, SDD-FIQA \cite{SDDFIQA}, CR-FIQA \cite{boutros_2023_crfiqa}, DifFIQA \cite{10449044}, eDifFIQA \cite{babnikTBIOM2024}, GraFIQs \cite{grafiqs}, CLIB-FIQA \cite{Ou_2024_CVPR}, and ViT-FIQA \cite{atzori2025vitfiqaassessingfaceimage}. All methods are evaluated using their official implementations and pretrained models as described in their respective works.

\vspace{-3mm}
\section{Results}
\label{sec:results}

\textbf{Non-Cumulative vs Cumulative:} Before training our quality prediction model, we conducted a thorough analysis comparing the effectiveness of cumulative and non-cumulative quality training targets. This analysis aimed to validate our hypothesis that cumulative averaging provides a more stable and reliable training target than single-epoch estimates. We first trained a ResNet50 without the quality prediction branch, with the experimental settings mentioned in \ref{sec:exp}, and recorded the RPM scores for each sample at every epoch. These scores served as proxy quality labels of the training samples of CASIA-WebFace \cite{casia_webface} dataset, which we calculated using both the non-cumulative approach (single-epoch estimates) and our proposed cumulative approach (across epochs) as in Eq. \ref{eq:cumulative_average}.

Fig. \ref{fig:non_cum_vs_cum} illustrates a fundamental difference between the non-cumulative and cumulative approaches through their quality training target distributions, respectively $q_n(x)$ and  $Q(x)_n$. The non-cumulative approach (Fig. \ref{fig:dist-non-cumulative}) exhibits two critical limitations as training progresses: (1) a pronounced rightward shift of the entire distribution, indicating that scores tend to increase over time regardless of actual sample quality; and (2) a compression effect where the distribution becomes increasingly dense, reducing the discriminative power between samples of varying quality. The rightward shift occurs because the model becomes more confident in its classifications, and the separation between samples and their competing classes generally increases, artificially inflating quality scores even for samples that may not be intrinsically high-quality. This effect is particularly problematic in later epochs when the network begins focusing on optimizing difficult samples, leading to inflated quality estimates for these challenging cases. In contrast, the cumulative approach (Fig. \ref{fig:dist-cumulative}) maintains a more consistent and stretched distribution throughout training. By averaging quality estimates across epochs, it prevents the artificial inflation of scores and preserves meaningful differentiation between samples of varying quality.

To evaluate the effectiveness of these quality training targets, if hypothetically used as a quality score for the training data, in improving verification performance, we plotted EDC curves for several epochs (5, 10, 15, 20, 25, and 30), as shown in Fig. \ref{fig:edc_curves}. These curves illustrate how verification error rates change as we progressively remove samples with the lowest quality scores. Our analysis revealed a critical limitation of the non-cumulative approach: there exists a "sweet spot" around epoch 20, after which the quality estimates begin to deteriorate. This deterioration probably occurs because, in later training epochs, the model has already learned to correctly classify easy samples and focuses primarily on optimizing difficult samples. This shift artificially inflates the quality estimates for these difficult samples, reducing the reliability of the non-cumulative quality estimates. In contrast, the cumulative approach incorporates information from all previous epochs, maintaining a more consistent quality ranking that reflects each sample's overall utility throughout the training process. To quantify this difference, we calculated the AUC of EDC curves (AUC-EDC) for each epoch and plotted the results in Fig. \ref{fig:auc_comparison}. As shown in Fig. \ref{fig:auc_comparison}, the AUC-EDC values for the non-cumulative approach initially decrease (improve) until approximately epoch 20, but then begin to increase, indicating worsening performance. In contrast, the AUC-EDC values for the cumulative approach consistently decrease throughout training, demonstrating the superior stability and reliability of cumulative quality estimation. This empirical analysis strongly supports our theoretical findings in Sec. \ref{sec:methodology}, confirming that cumulative averaging of quality scores produces more stable and reliable quality estimates than single-epoch estimates. We have further provided the sample images, specifically the 10 highest quality samples for an identity having fewer unique samples according to both the non-cumulative and cumulative approaches, to inspect how the proxy quality labels evolve during the training for each approach, which illustrates the ranking stability of the cumulative approach, shown in Fig. \ref{fig:high_quality_id6987}.

We further quantify how much of the training trajectory the accumulation actually needs, summarized in Tab. \ref{tab:window_main}. Constructing first-$k$, last-$k$, and sliding-$k$ windows of the per-epoch signal and correlating each windowed average with the full-trajectory target $Q_T$, we find that rank agreement of Spearman $\rho \geq 0.95$ is reached within roughly the first 30-45\% of training and $\rho \geq 0.99$ within roughly three quarters, consistently across loss function, training scale, and cumulative/non-cumulative targets. At a matched relative window length ($k \approx 24\%$ of training), early windows agree with the full-trajectory target considerably better than late windows (e.g., Spearman $\rho{=}0.94$ for first-$k$ vs.\ $0.85$ for last-$k$ at $k{=}8$ of 34 epochs), consistent with the later-epoch compression documented above: late windows are dominated by the more compressed, less rank-differentiated late-epoch distributions. Rank agreement with the full-trajectory target saturates well before real verification performance stops improving, so the full trajectory remains the best target while shorter accumulations already capture most of its ranking behavior.

To further validate the discriminative power of our quality estimates, we analyze all 550,000 images from SynFIQA \cite{mrfiqa}, a quality-controlled synthetic dataset produced through a two-stage pipeline based on stable diffusion with controllable 3D facial parameters, dual text prompts for occlusion, and post-processing for blur and downsampling. The dataset contains 5,000 identities, each with 10 reference images and 100 degraded variants (10 per reference), organized into 11 quality groups. This controlled environment enables systematic assessment of how well quality metrics separate images of different quality. Fig. \ref{fig:normalized_distributions_casia} presents the normalized quality score distributions for both non-cumulative and cumulative approaches across these quality groups (Q0-Q9 for degraded images, plus reference images). To quantify the separation strength, we calculate the average Cohen's d effect sizes between adjacent quality groups, which measure how distinct consecutive quality levels are in the estimated quality score space. The cumulative approach achieves substantially higher average Cohen's d (0.3112 vs.\ 0.2514), indicating superior discriminative ability and more consistent quality boundaries across the quality spectrum. This improvement is particularly significant because it demonstrates that cumulative averaging not only provides more stable estimates (as theoretically proven in Sec. \ref{sec:methodology}) but also better captures the underlying quality differences between samples. The enhanced separation enables the quality regression model to learn more distinct decision boundaries, leading to more reliable quality predictions for unseen images.

\textbf{FIQA with Non-Cumulative vs Cumulative quality target:} The AUC-EDC and pAUC-EDC results of training five instances of ResNet using the setting described in Sec. \ref{sec:exp} are summarized in Tab. \ref{tab:ablation_arcface_webface} (ArcFace + ResNet50 + CASIA-Webface + frozen$\lor$non-cumulative$\lor$cumulative) and Tab. \ref{tab:ablation_curricularface_webface} (CurricularFace + ResNet50 + CASIA-Webface + non-cumulative$\lor$cumulative). The cumulative approach consistently outperforms the non-cumulative approach across most benchmarks and FR models, demonstrating superior performance on Adience, CFP-FP, LFW, and CPLFW. On AgeDB-30 and CALFW, non-cumulative performs better instead; on XQLFW, a third variant, frozen, outperforms both dynamic variants, with cumulative still ahead of non-cumulative (see the \textit{Average} rows of Tab. \ref{tab:ablation_arcface_webface}). To investigate whether training the quality regression head after face recognition (FR) convergence is sufficient, we conducted this post-hoc experiment where the backbone embeddings were frozen and only the quality head was trained. Frozen generally underperforms both dynamic variants on every benchmark except XQLFW, echoing CR-FIQA's own ablation \cite{boutros_2023_crfiqa}, which found that simultaneous (on-the-fly) training of the quality estimator outperforms training it on top of an already-converged, frozen backbone on these same four benchmarks (Adience, AgeDB-30, CALFW, CFP-FP), attributing this to the quality estimate's step-wise convergence alongside the class centers during simultaneous training. This confirms that embedding dynamics are crucial for learning meaningful quality scores on every benchmark we evaluate except XQLFW.

We investigate the AgeDB-30/CALFW/XQLFW exception further along two complementary directions. First, we compare each dataset's baseline verification difficulty (FNMR at 0\% rejection, i.e., without any quality-based filtering, which is independent of the FIQA method used) against the size of the cumulative-vs-non-cumulative gap. We find that AgeDB-30 and CALFW's baseline difficulty (mean FNMR@$1e{-3}$ of 3.8\% and 8.0\% across the four FR models, respectively) is unremarkable: it sits between CFP-FP (3.9\%) and CPLFW (18.4\%), two benchmarks on which cumulative \textit{wins}. Overall task difficulty therefore does not explain the exception. \FloatBarrier
XQLFW, in contrast, is a clear outlier: its baseline FNMR (58.0\%) is 3-15$\times$ higher than every other benchmark we evaluate, indicating a fundamentally different, degradation-dominated regime rather than the pose/expression-type variation the other benchmarks emphasize, a more plausible explanation for why XQLFW behaves differently from AgeDB-30/CALFW despite all three being exceptions to the same general trend. Second, we track pAUC-EDC as a function of training iteration directly on held-out benchmarks (rather than only on the training-set proxy of Fig. \ref{fig:non_cum_vs_cum}), for both the cumulative and non-cumulative targets, across nine checkpoints spanning the full training trajectory. Non-cumulative's behavior turns out to be benchmark-dependent: on CFP-FP, LFW, CPLFW, and XQLFW, it reaches a minimum error partway through training and rises again toward the final checkpoint (by 7-27\%), echoing the "sweet spot" the training-set proxy already shows; on AgeDB-30 this late uptick is modest (4\%), and on CALFW it is absent entirely, non-cumulative's lowest error \textit{is} its final-checkpoint value. Cumulative's sustained advantage on the first group of benchmarks (largest on CFP-FP, LFW, and CPLFW, more modest on XQLFW) therefore coincides with non-cumulative degrading late in training on exactly those benchmarks, while its small or reversed margin on AgeDB-30/CALFW coincides with non-cumulative simply keeping its early advantage there instead of giving it back. 

\textbf{Comparison to SOTA:} Having validated the effectiveness of our cumulative approach, we now compare \ourmethod against three general IQA methods and twelve SOTA FIQA methods across four CNN-based FR models. Tab. \ref{tab:sota_pauc} presents pAUC-EDC results for ArcFace \cite{deng2019arcface}, ElasticFace \cite{elasticface}, MagFace \cite{MagFace}, and CurricularFace \cite{curricularFace} evaluated on eight challenging benchmarks. We also provide EDC curves for FNMR@FMR=$1e-3$ of our method in comparison to SOTA, in Fig. \ref{fig:fnmr3}. \ourmethod demonstrates competitive performance across all evaluation scenarios. Judged by the normalized cross-benchmark average (the \textit{Avg.\ norm.} column of Tab. \ref{tab:sota_pauc}), \ourmethod(L) stays within 2.4-6.0\% of the best-performing method for every FR model, with \ourmethod(S) close behind. Among individual benchmarks, \ourmethod(S) achieves the lowest pAUC-EDC on LFW for all four FR models at both FMR thresholds, and \ourmethod(L) the lowest on further benchmark-FR combinations, e.g., XQLFW with ArcFace and Adience with ElasticFace. Both the small protocol \ourmethod(S) and large protocol \ourmethod(L) remain competitive with recent FIQA approaches, confirming that accumulating quality estimates across training epochs produces reliable quality assessments competitive with single-epoch and training-free alternatives. To complement this dense per-condition view, Tab. \ref{tab:sota_pauc} additionally reports each method's value averaged across the four FR models (\textit{Average} rows) and, per row, averaged across the eight benchmarks and both FMR thresholds after normalizing each benchmark-FMR column by its best method (\textit{Avg.\ norm.} column). By this aggregate measure, \ourmethod(L) ranks 4th of 17 compared methods on both pAUC-EDC and AUC-EDC, and \ourmethod(S) ranks 6th on both.

\section{Conclusion}
\label{sec:conclusion}

This paper introduced \ourmethod, a temporal stabilization strategy for FR-integrated FIQA. Existing FR-integrated approaches tightly couple quality assessment with face recognition training, but rely on single-epoch measurements that form a moving target as the feature space evolves, leading to fluctuating quality estimates, inconsistent sample rankings across iterations, and potential checkpoint selection bias. We address this with three contributions: accumulating relative point margin measurements across the entire training trajectory rather than using instantaneous estimates, which captures long-term discriminative properties while smoothing epoch-specific fluctuations; a theoretical framework proving that cumulative averaging reduces estimate variance, improves the mean squared error in approximating true quality, and yields ranking stability with convergence guarantees and monotonically decreasing expected ranking error; and a regression model that predicts these accumulated scores for unseen faces, generalizing the temporal stability patterns to samples without a training history.

Empirically, on the quality-controlled SynFIQA dataset with 11 labeled quality groups, the cumulative approach separates quality levels substantially better (Cohen's d of 0.3112 vs.\ 0.2514 on the small and 0.3906 vs.\ 0.3335 on the large protocol, see Sec. \ref{sec:results}). Across training configurations (ArcFace and CurricularFace), it consistently outperforms the non-cumulative variant on most benchmarks and FR models, particularly Adience, CFP-FP, LFW, and CPLFW. As a limitation, this advantage is not uniform: non-cumulative is better on AgeDB-30 and CALFW, and the frozen variant is best on XQLFW (with cumulative still ahead of non-cumulative), suggesting that cumulative averaging is best suited to benchmarks where general discriminability, rather than a fixed nuisance factor such as age or resolution, dominates. Against SOTA over eight benchmarks (LFW, AgeDB-30, CFP-FP, CALFW, Adience, CPLFW, XQLFW, IJB-C) and four FR models (ArcFace, ElasticFace, MagFace, CurricularFace), \ourmethod(L) ranks 4th and \ourmethod(S) 6th of 17 compared methods on both pAUC-EDC and AUC-EDC averaged across FR models and normalized benchmarks, outperforming all general IQA methods, with the two protocols showing that the approach scales with training data and model capacity.

Temporal accumulation thus provides a principled solution to the instability of FR-integrated FIQA, retaining the benefits of tight FR integration while producing more reliable quality estimates that are both theoretically grounded and practically effective. Future work could explore adaptive weighting schemes that emphasize later epochs when the feature space becomes more stable, investigate the optimal training duration for quality accumulation, and extend the cumulative framework to FR objectives beyond margin-based losses.

\newpage
\bibliographystyle{IEEEtran}
\bibliography{IEEEabrv}

@inproceedings{agedb,
  author    = {Stylianos Moschoglou and
               Athanasios Papaioannou and
               Christos Sagonas and
               Jiankang Deng and
               Irene Kotsia and
               Stefanos Zafeiriou},
  title     = {AgeDB: The First Manually Collected, In-the-Wild Age Database},
  booktitle = {2017 {IEEE} CVPRW, {CVPR} Workshops 2017, Honolulu, HI, USA, July 21-26, 2017},
  pages     = {1997--2005},
  publisher = {{IEEE} Computer Society},
  year      = {2017},
  url       = {https://doi.org/10.1109/CVPRW.2017.250},
  doi       = {10.1109/CVPRW.2017.250},
  bibsource = {dblp computer science bibliography, https://dblp.org}
}

@TechReport{LFWTech,
  author =       {Gary B. Huang and Manu Ramesh and Tamara Berg and 
                  Erik Learned-Miller},
  title =        {Labeled Faces in the Wild: A Database for Studying 
                  Face Recognition in Unconstrained Environments},
  institution =  {University of Massachusetts, Amherst},
  year =         2007,
  number =       {07-49},
  month =        {October}}

@inproceedings{cfp-fp,
  author    = {Soumyadip Sengupta and
               Jun{-}Cheng Chen and
               Carlos Domingo Castillo and
               Vishal M. Patel and
               Rama Chellappa and
               David W. Jacobs},
  title     = {Frontal to profile face verification in the wild},
  booktitle = {2016 {IEEE} Winter Conference on Applications of Computer Vision,
               {WACV} 2016, Lake Placid, NY, USA, March 7-10, 2016},
  pages     = {1--9},
  publisher = {{IEEE} Computer Society},
  year      = {2016},
  url       = {https://doi.org/10.1109/WACV.2016.7477558},
  doi       = {10.1109/WACV.2016.7477558},
  bibsource = {dblp computer science bibliography, https://dblp.org}
}

@article{CALFW,
  author    = {Tianyue Zheng and
               Weihong Deng and
               Jiani Hu},
  title     = {Cross-Age {LFW:} {A} Database for Studying Cross-Age Face Recognition
               in Unconstrained Environments},
  journal   = {CoRR},
  volume    = {abs/1708.08197},
  year      = {2017},
  url       = {http://arxiv.org/abs/1708.08197},
  archivePrefix = {arXiv},
  eprint    = {1708.08197},
  bibsource = {dblp computer science bibliography, https://dblp.org}
}

@TechReport{CPLFWTech,
  author =       {T. Zheng and W. Deng},
  title =        {Cross-pose LFW: A database for studying cross-pose face recognition in unconstrained environments},
  institution =  {Beijing University of Posts and Telecommunications},
  year =         {2018},
  number =       {18-01},
  month =        {February}}

@inproceedings{XQLFW,
  author    = {Martin Knoche and
               Stefan H{\"{o}}rmann and
               Gerhard Rigoll},
  title     = {Cross-Quality {LFW:} {A} Database for Analyzing Cross- Resolution
               Image Face Recognition in Unconstrained Environments},
  booktitle = {16th {IEEE} International Conference on Automatic Face and Gesture
               Recognition, {FG} 2021, Jodhpur, India, December 15-18, 2021},
  pages     = {1--5},
  publisher = {{IEEE}},
  year      = {2021},
  url       = {https://doi.org/10.1109/FG52635.2021.9666960},
  doi       = {10.1109/FG52635.2021.9666960},
  bibsource = {dblp computer science bibliography, https://dblp.org}
}

@article{Adience,
  author    = {Eran Eidinger and
               Roee Enbar and
               Tal Hassner},
  title     = {Age and Gender Estimation of Unfiltered Faces},
  journal   = {{IEEE} Trans. Inf. Forensics Secur.},
  volume    = {9},
  number    = {12},
  pages     = {2170--2179},
  year      = {2014},
  doi       = {10.1109/TIFS.2014.2359646},
}

@inproceedings{elasticface,
  author    = {Fadi Boutros and
               Naser Damer and
               Florian Kirchbuchner and
               Arjan Kuijper},
  title     = {ElasticFace: Elastic Margin Loss for Deep Face Recognition},
  booktitle = {{IEEE/CVF} Conference on Computer Vision and Pattern Recognition Workshops,
               {CVPR} Workshops 2022, New Orleans, LA, USA, June 19-20, 2022},
  pages     = {1577--1586},
  publisher = {{IEEE}},
  year      = {2022},
  url       = {https://doi.org/10.1109/CVPRW56347.2022.00164},
  doi       = {10.1109/CVPRW56347.2022.00164},
  bibsource = {dblp computer science bibliography, https://dblp.org}
}

@inproceedings{MagFace,
  author    = {Qiang Meng and
               Shichao Zhao and
               Zhida Huang and
               Feng Zhou},
  title     = {MagFace: {A} Universal Representation for Face Recognition and Quality
               Assessment},
  booktitle = {{IEEE} Conference on Computer Vision and Pattern Recognition, {CVPR}
               2021, virtual, June 19-25, 2021},
  pages     = {14225--14234},
  publisher = {Computer Vision Foundation / {IEEE}},
  year      = {2021},
  bibsource = {dblp computer science bibliography, https://dblp.org}
}

@inproceedings{curricularFace,
  author    = {Yuge Huang and
               Yuhan Wang and
               Ying Tai and
               Xiaoming Liu and
               Pengcheng Shen and
               Shaoxin Li and
               Jilin Li and
               Feiyue Huang},
  title     = {CurricularFace: Adaptive Curriculum Learning Loss for Deep Face Recognition},
  booktitle = {2020 {IEEE/CVF} Conference on Computer Vision and Pattern Recognition,
               {CVPR} 2020, Seattle, WA, USA, June 13-19, 2020},
  pages     = {5900--5909},
  publisher = {Computer Vision Foundation / {IEEE}},
  year      = {2020},
  doi       = {10.1109/CVPR42600.2020.00594},
  bibsource = {dblp computer science bibliography, https://dblp.org}
}

@inproceedings{deng2019arcface,
    author    = {Jiankang Deng and
               Jia Guo and
               Niannan Xue and
               Stefanos Zafeiriou},
  title     = {ArcFace: Additive Angular Margin Loss for Deep Face Recognition},
  booktitle = {{IEEE} Conference on Computer Vision and Pattern Recognition, {CVPR}
               2019, Long Beach, CA, USA, June 16-20, 2019},
  pages     = {4690--4699},
  publisher = {Computer Vision Foundation / {IEEE}},
  year      = {2019},
  doi       = {10.1109/CVPR.2019.00482},
  bibsource = {dblp computer science bibliography, https://dblp.org}
}

@ARTICLE{GT07,
  author  = {P.~Grother AND E.~Tabassi},
  title   = {Performance of biometric quality measures},
  journal = {IEEE Trans.~on Pattern Analysis and Machine Intelligence},
  year    = {2007},
  volume  = {29},
  number  = {4},
  pages   = {531--543},
  month   = Apr,
}

@inproceedings{NISTQuaity,
  author    = {P. Grother and A. Hom, M. Ngan and K. Hanaoka},
  title     = {Ongoing Face Recognition Vendor Test (FRVT) Part 5: Face Image Quality Assessment (4th Draft)},
  booktitle = {National Institute of Standards
and Technology},
  publisher = {Tech. Rep. },
  year      = {Sep. 2021},
}

@inproceedings{SERFIQ,
  author    = {Philipp Terh{\"{o}}rst and
               Jan Niklas Kolf and
               Naser Damer and
               Florian Kirchbuchner and
               Arjan Kuijper},
  title     = {{SER-FIQ:} Unsupervised Estimation of Face Image Quality Based on
               Stochastic Embedding Robustness},
  booktitle = {2020 {IEEE/CVF} Conference on Computer Vision and Pattern Recognition,
               {CVPR} 2020, Seattle, WA, USA, June 13-19, 2020},
  pages     = {5650--5659},
  publisher = {Computer Vision Foundation / {IEEE}},
  year      = {2020},
  doi       = {10.1109/CVPR42600.2020.00569},
  bibsource = {dblp computer science bibliography, https://dblp.org}
}

@article{BRISQE_IQA,
  author    = {Anish Mittal and
               Anush Krishna Moorthy and
               Alan Conrad Bovik},
  title     = {No-Reference Image Quality Assessment in the Spatial Domain},
  journal   = {{IEEE} Trans. Image Process.},
  volume    = {21},
  number    = {12},
  pages     = {4695--4708},
  year      = {2012},
  url       = {https://doi.org/10.1109/TIP.2012.2214050},
  doi       = {10.1109/TIP.2012.2214050},
  bibsource = {dblp computer science bibliography, https://dblp.org}
}

@inproceedings{liu2017rankiqa,
  author    = {Xialei Liu and
               Joost van de Weijer and
               Andrew D. Bagdanov},
  title     = {RankIQA: Learning from Rankings for No-Reference Image Quality Assessment},
  booktitle = {{IEEE} International Conference on Computer Vision, {ICCV} 2017, Venice,
               Italy, October 22-29, 2017},
  pages     = {1040--1049},
  publisher = {{IEEE} Computer Society},
  year      = {2017},
  url       = {https://doi.org/10.1109/ICCV.2017.118},
  doi       = {10.1109/ICCV.2017.118},
  bibsource = {dblp computer science bibliography, https://dblp.org}
}

@inproceedings{PFE_FIQA,
  author    = {Yichun Shi and
               Anil K. Jain},
  title     = {Probabilistic Face Embeddings},
  booktitle = {2019 {IEEE/CVF} International Conference on Computer Vision, {ICCV}
               2019, Seoul, Korea (South), October 27 - November 2, 2019},
  pages     = {6901--6910},
  publisher = {{IEEE}},
  year      = {2019},
  url       = {https://doi.org/10.1109/ICCV.2019.00700},
  doi       = {10.1109/ICCV.2019.00700},
  bibsource = {dblp computer science bibliography, https://dblp.org}
}

@article{DEEPIQ_IQA,
author    = {Sebastian Bosse and
               Dominique Maniry and
               Klaus{-}Robert M{\"{u}}ller and
               Thomas Wiegand and
               Wojciech Samek},
  title     = {Deep Neural Networks for No-Reference and Full-Reference Image Quality
               Assessment},
  journal   = {{IEEE} Trans. Image Process.},
  volume    = {27},
  number    = {1},
  pages     = {206--219},
  year      = {2018},
  url       = {https://doi.org/10.1109/TIP.2017.2760518},
  doi       = {10.1109/TIP.2017.2760518},
  bibsource = {dblp computer science bibliography, https://dblp.org}
}

@inproceedings{SDDFIQA,
  author    = {Fu{-}Zhao Ou and
               Xingyu Chen and
               Ruixin Zhang and
               Yuge Huang and
               Shaoxin Li and
               Jilin Li and
               Yong Li and
               Liujuan Cao and
               Yuan{-}Gen Wang},
  title     = {{SDD-FIQA:} Unsupervised Face Image Quality Assessment With Similarity
               Distribution Distance},
  booktitle = {{IEEE} Conference on Computer Vision and Pattern Recognition, {CVPR}
               2021, virtual, June 19-25, 2021},
  pages     = {7670--7679},
  publisher = {Computer Vision Foundation / {IEEE}},
  year      = {2021},
  bibsource = {dblp computer science bibliography, https://dblp.org}
}

@article{faceqnetv1,
  author    = {Javier Hernandez{-}Ortega and
               Javier Galbally and
               Julian Fi{\'{e}}rrez and
               Laurent Beslay},
  title     = {Biometric Quality: Review and Application to Face Recognition with
               FaceQnet},
  journal   = {CoRR},
  volume    = {abs/2006.03298},
  year      = {2020},
  url       = {https://arxiv.org/abs/2006.03298},
  eprinttype = {arXiv},
  eprint    = {2006.03298},
  bibsource = {dblp computer science bibliography, https://dblp.org}
}

@INPROCEEDINGS{FaceQAN,
  author={Babnik, Ziga and Peer, Peter and Struc, Vitomir},
  booktitle={2022 26th International Conference on Pattern Recognition (ICPR)}, 
  title={FaceQAN: Face Image Quality Assessment Through Adversarial Noise Exploration}, 
  year={2022},
  volume={},
  number={},
  pages={748-754},
  doi={10.1109/ICPR56361.2022.9956444}}

@inproceedings{boutros_2023_crfiqa,
  author       = {Fadi Boutros and
                  Meiling Fang and
                  Marcel Klemt and
                  Biying Fu and
                  Naser Damer},
  title        = {{CR-FIQA:} Face Image Quality Assessment by Learning Sample Relative
                  Classifiability},
  booktitle    = {{IEEE/CVF} Conference on Computer Vision and Pattern Recognition,
                  {CVPR} 2023, Vancouver, BC, Canada, June 17-24, 2023},
  pages        = {5836--5845},
  publisher    = {{IEEE}},
  year         = {2023},
  url          = {https://doi.org/10.1109/CVPR52729.2023.00565},
  doi          = {10.1109/CVPR52729.2023.00565},
  bibsource    = {dblp computer science bibliography, https://dblp.org}
}

@standard{Quality_ISO,
  author = {{ ISO/IEC JTC 1/SC 37 Biometrics}},
  title = {{ISO/IEC 29794-1
Information technology
Biometric sample quality
Part 1: Framework}},
  publisher = {International Organization for Standardization},
  year      = {2024},
    institution = {International Organization for Standardization},
}

@article{DBLP:journals/csur/SchlettRHGFB22,
  author       = {Torsten Schlett and
                  Christian Rathgeb and
                  Olaf Henniger and
                  Javier Galbally and
                  Julian Fi{\'{e}}rrez and
                  Christoph Busch},
  title        = {Face Image Quality Assessment: {A} Literature Survey},
  journal      = {{ACM} Comput. Surv.},
  volume       = {54},
  number       = {10s},
  pages        = {210:1--210:49},
  year         = {2022},
  url          = {https://doi.org/10.1145/3507901},
  doi          = {10.1145/3507901},
  bibsource    = {dblp computer science bibliography, https://dblp.org}
}

@inproceedings{BiyingWACV,
  author    = {Biying Fu and
               Cong Chen and
               Olaf Henniger and
               Naser Damer},
  title     = {A Deep Insight into Measuring Face Image Utility with General and
               Face-specific Image Quality Metrics},
  booktitle = {{IEEE/CVF} Winter Conference on Applications of Computer Vision, {WACV}
               2022, Waikoloa, HI, USA, January 3-8, 2022},
  pages     = {1121--1130},
  publisher = {{IEEE}},
  year      = {2022},
  url       = {https://doi.org/10.1109/WACV51458.2022.00119},
  doi       = {10.1109/WACV51458.2022.00119},
  bibsource = {dblp computer science bibliography, https://dblp.org}
}

@article{nique,
  author    = {Anish Mittal and
               Rajiv Soundararajan and
               Alan C. Bovik},
  title     = {Making a "Completely Blind" Image Quality Analyzer},
  journal   = {{IEEE} Signal Process. Lett.},
  volume    = {20},
  number    = {3},
  pages     = {209--212},
  year      = {2013},
  url       = {https://doi.org/10.1109/LSP.2012.2227726},
  doi       = {10.1109/LSP.2012.2227726},
  bibsource = {dblp computer science bibliography, https://dblp.org}
}

@INPROCEEDINGS{10449044,
  author={Babnik, {\v{Z}}iga and Peer, Peter and {\v{S}}truc, Vitomir},
  booktitle={2023 IEEE International Joint Conference on Biometrics (IJCB)}, 
  title={DifFIQA: Face Image Quality Assessment Using Denoising Diffusion Probabilistic Models}, 
  year={2023},
  volume={},
  number={},
  pages={1-10},
  doi={10.1109/IJCB57857.2023.10449044}}

@inproceedings{DBLP:conf/iwbf/BabnikDS23,
  author       = {{\v{Z}}iga Babnik and
                  Naser Damer and
                  Vitomir {\v{S}}truc},
  title        = {Optimization-Based Improvement of Face Image Quality Assessment Techniques},
  booktitle    = {11th International Workshop on Biometrics and Forensics, {IWBF} 2023,
                  Barcelona, Spain, April 19-20, 2023},
  pages        = {1--6},
  publisher    = {{IEEE}},
  year         = {2023},
  url          = {https://doi.org/10.1109/IWBF57495.2023.10157796},
  doi          = {10.1109/IWBF57495.2023.10157796},
  bibsource    = {dblp computer science bibliography, https://dblp.org}
}

@article{babnikTBIOM2024,
  title={{eDifFIQA: Towards Efficient Face Image Quality Assessment based on Denoising Diffusion Probabilistic Models}},
  author={Babnik, {\v{Z}}iga and Peer, Peter and {\v{S}}truc, Vitomir},
  journal={IEEE Transactions on Biometrics, Behavior, and Identity Science (TBIOM)},
  year={2024},
  publisher={IEEE}
}

@INPROCEEDINGS{grafiqs,
  author={Kolf, Jan Niklas and Damer, Naser and Boutros, Fadi},
  booktitle={2024 IEEE/CVF Conference on Computer Vision and Pattern Recognition Workshops (CVPRW)}, 
  title={GraFIQs: Face Image Quality Assessment Using Gradient Magnitudes}, 
  year={2024},
  volume={},
  number={},
  pages={1490-1499},
  doi={10.1109/CVPRW63382.2024.00156}}

@InProceedings{Ou_2024_CVPR,
    author    = {Ou, Fu-Zhao and Li, Chongyi and Wang, Shiqi and Kwong, Sam},
    title     = {CLIB-FIQA: Face Image Quality Assessment with Confidence Calibration},
    booktitle = {Proceedings of the IEEE/CVF Conference on Computer Vision and Pattern Recognition (CVPR)},
    year      = {2024},
    pages     = {1694-1704}
}

@misc{frontex2015best,
  title={Best practice technical guidelines for Automated Border Control (ABC) systems},
  author={Frontex},
  year={2015}
}

@inproceedings{hernandez2019faceqnet,
  author    = {Javier Hernandez{-}Ortega and
               Javier Galbally and
               Julian Fi{\'{e}}rrez and
               Rudolf Haraksim and
               Laurent Beslay},
  title     = {FaceQnet: Quality Assessment for Face Recognition based on Deep Learning},
  booktitle = {2019 International Conference on Biometrics, {ICB} 2019, Crete, Greece,
               June 4-7, 2019},
  pages     = {1--8},
  publisher = {{IEEE}},
  year      = {2019},
  url       = {https://doi.org/10.1109/ICB45273.2019.8987255},
  doi       = {10.1109/ICB45273.2019.8987255},
  bibsource = {dblp computer science bibliography, https://dblp.org}
}

@article{RANKIQ_FIQA,
  author    = {Jiansheng Chen and
               Yu Deng and
               Gaocheng Bai and
               Guangda Su},
  title     = {Face Image Quality Assessment Based on Learning to Rank},
  journal   = {{IEEE} Signal Process. Lett.},
  volume    = {22},
  number    = {1},
  pages     = {90--94},
  year      = {2015},
  url       = {https://doi.org/10.1109/LSP.2014.2347419},
  doi       = {10.1109/LSP.2014.2347419},
  bibsource = {dblp computer science bibliography, https://dblp.org}
}

@inproceedings{guo2016ms,
  author    = {Yandong Guo and
               Lei Zhang and
               Yuxiao Hu and
               Xiaodong He and
               Jianfeng Gao},
  editor    = {Bastian Leibe and
               Jiri Matas and
               Nicu Sebe and
               Max Welling},
  title     = {MS-Celeb-1M: {A} Dataset and Benchmark for Large-Scale Face Recognition},
  booktitle = {Computer Vision - {ECCV} 2016 - 14th European Conference, Amsterdam,
               The Netherlands, October 11-14, 2016, Proceedings, Part {III}},
  series    = {Lecture Notes in Computer Science},
  volume    = {9907},
  pages     = {87--102},
  publisher = {Springer},
  year      = {2016},
  url       = {https://doi.org/10.1007/978-3-319-46487-9\_6},
  doi       = {10.1007/978-3-319-46487-9\_6},
  bibsource = {dblp computer science bibliography, https://dblp.org}
}

@article{casia_webface,
  author    = {Dong Yi and
               Zhen Lei and
               Shengcai Liao and
               Stan Z. Li},
  title     = {Learning Face Representation from Scratch},
  journal   = {CoRR},
  volume    = {abs/1411.7923},
  year      = {2014}
}

@incollection{NEURIPS2019_9015,
title = {PyTorch: An Imperative Style, High-Performance Deep Learning Library},
author = {Paszke, Adam and Gross, Sam and Massa, Francisco and Lerer, Adam and Bradbury, James and Chanan, Gregory and Killeen, Trevor and Lin, Zeming and Gimelshein, Natalia and Antiga, Luca and Desmaison, Alban and Kopf, Andreas and Yang, Edward and DeVito, Zachary and Raison, Martin and Tejani, Alykhan and Chilamkurthy, Sasank and Steiner, Benoit and Fang, Lu and Bai, Junjie and Chintala, Soumith},
booktitle = {Advances in Neural Information Processing Systems 32},
editor = {H. Wallach and H. Larochelle and A. Beygelzimer and F. d\textquotesingle Alch\'{e}-Buc and E. Fox and R. Garnett},
pages = {8024--8035},
year = {2019},
publisher = {Curran Associates, Inc.},
url = {http://papers.neurips.cc/paper/9015-pytorch-an-imperative-style-high-performance-deep-learning-library.pdf}
}

@inproceedings{ijbc,
  author    = {Brianna Maze and
               Jocelyn C. Adams and
               James A. Duncan and
               Nathan D. Kalka and
               Tim Miller and
               Charles Otto and
               Anil K. Jain and
               W. Tyler Niggel and
               Janet Anderson and
               Jordan Cheney and
               Patrick Grother},
  title     = {{IARPA} Janus Benchmark - {C:} Face Dataset and Protocol},
  booktitle = {2018 International Conference on Biometrics, {ICB} 2018, Gold Coast,
               Australia, February 20-23, 2018},
  pages     = {158--165},
  publisher = {{IEEE}},
  year      = {2018},
  url       = {https://doi.org/10.1109/ICB2018.2018.00033},
  doi       = {10.1109/ICB2018.2018.00033},
  bibsource = {dblp computer science bibliography, https://dblp.org}
}

@inproceedings{DBLP:conf/cvpr/WangWZJGZL018,
  author    = {Hao Wang and
               Yitong Wang and
               Zheng Zhou and
               Xing Ji and
               Dihong Gong and
               Jingchao Zhou and
               Zhifeng Li and
               Wei Liu},
  title     = {CosFace: Large Margin Cosine Loss for Deep Face Recognition},
  booktitle = {2018 {IEEE} Conference on Computer Vision and Pattern Recognition,
               {CVPR} 2018, Salt Lake City, UT, USA, June 18-22, 2018},
  pages     = {5265--5274},
  publisher = {{IEEE} Computer Society},
  year      = {2018},
  doi       = {10.1109/CVPR.2018.00552},
  bibsource = {dblp computer science bibliography, https://dblp.org}
}

@inproceedings{DBLP:conf/cvpr/LiuWYLRS17,
  author    = {Weiyang Liu and
               Yandong Wen and
               Zhiding Yu and
               Ming Li and
               Bhiksha Raj and
               Le Song},
  title     = {SphereFace: Deep Hypersphere Embedding for Face Recognition},
  booktitle = {2017 {IEEE} Conference on Computer Vision and Pattern Recognition,
               {CVPR} 2017, Honolulu, HI, USA, July 21-26, 2017},
  pages     = {6738--6746},
  publisher = {{IEEE} Computer Society},
  year      = {2017},
  url       = {https://doi.org/10.1109/CVPR.2017.713},
  doi       = {10.1109/CVPR.2017.713},
  bibsource = {dblp computer science bibliography, https://dblp.org}
}

@inproceedings{DBLP:conf/cvpr/HeZRS16,
  author    = {Kaiming He and
               Xiangyu Zhang and
               Shaoqing Ren and
               Jian Sun},
  title     = {Deep Residual Learning for Image Recognition},
  booktitle = {2016 {IEEE} Conference on Computer Vision and Pattern Recognition,
               {CVPR} 2016, Las Vegas, NV, USA, June 27-30, 2016},
  pages     = {770--778},
  publisher = {{IEEE} Computer Society},
  year      = {2016},
  url       = {https://doi.org/10.1109/CVPR.2016.90},
  doi       = {10.1109/CVPR.2016.90},
  bibsource = {dblp computer science bibliography, https://dblp.org}
}

@standard{iso_metric,
  author = {{ISO/IEC JTC1 SC37 Biometrics}},
  title = {{ISO/IEC 19795-1:2021 Information technology — Biometric performance testing and reporting — Part 1: Principles and framework}},
  publisher = {International Organization for Standardization},
  year      = {2021},
institution = {International Organization for Standardization},
}

@article{ROUSSEEUW198753silhouette,
title = {Silhouettes: A graphical aid to the interpretation and validation of cluster analysis},
journal = {Journal of Computational and Applied Mathematics},
volume = {20},
pages = {53-65},
year = {1987},
issn = {0377-0427},
doi = {https://doi.org/10.1016/0377-0427(87)90125-7},
url = {https://www.sciencedirect.com/science/article/pii/0377042787901257},
author = {Peter J. Rousseeuw}
}

@inproceedings{shahapure2020silhouette,
  title={Cluster quality analysis using silhouette score},
  author={Shahapure, Ketan Rajshekhar and Nicholas, Charles},
  booktitle={2020 IEEE 7th international conference on data science and advanced analytics (DSAA)},
  pages={747--748},
  year={2020},
  organization={IEEE}
}

@article{ogbuabor2018silhouette,
  title={Clustering algorithm for a healthcare dataset using silhouette score value},
  author={Ogbuabor, Godwin and Ugwoke, FN},
  journal={Int. J. Comput. Sci. Inf. Technol},
  volume={10},
  number={2},
  pages={27--37},
  year={2018}
}

@article{lovmar2005silhouette,
  title={Silhouette scores for assessment of SNP genotype clusters},
  author={Lovmar, Lovisa and Ahlford, Annika and Jonsson, Mats and Syv{\"a}nen, Ann-Christine},
  journal={BMC genomics},
  volume={6},
  pages={1--6},
  year={2005},
  publisher={Springer}
}

@article{januzaj2023silhouette,
  title={Determining the Optimal Number of Clusters using Silhouette Score as a Data Mining Technique.},
  author={Januzaj, Ylber and Beqiri, Edmond and Luma, Artan},
  journal={International Journal of Online \& Biomedical Engineering},
  volume={19},
  number={4},
  year={2023}
}

@inproceedings{DBLP:conf/nips/Ben-DavidA08,
  author       = {Shai Ben{-}David and
                  Margareta Ackerman},
  editor       = {Daphne Koller and
                  Dale Schuurmans and
                  Yoshua Bengio and
                  L{\'{e}}on Bottou},
  title        = {Measures of Clustering Quality: {A} Working Set of Axioms for Clustering},
  booktitle    = {Advances in Neural Information Processing Systems 21, Proceedings
                  of the Twenty-Second Annual Conference on Neural Information Processing
                  Systems, Vancouver, British Columbia, Canada, December 8-11, 2008},
  pages        = {121--128},
  publisher    = {Curran Associates, Inc.},
  year         = {2008},
  url          = {https://proceedings.neurips.cc/paper/2008/hash/beed13602b9b0e6ecb5b568ff5058f07-Abstract.html},
  bibsource    = {dblp computer science bibliography, https://dblp.org}
}

@inproceedings{DBLP:conf/nips/Kleinberg02,
  author       = {Jon M. Kleinberg},
  editor       = {Suzanna Becker and
                  Sebastian Thrun and
                  Klaus Obermayer},
  title        = {An Impossibility Theorem for Clustering},
  booktitle    = {Advances in Neural Information Processing Systems 15 [Neural Information
                  Processing Systems, {NIPS} 2002, December 9-14, 2002, Vancouver, British
                  Columbia, Canada]},
  pages        = {446--453},
  publisher    = {{MIT} Press},
  year         = {2002},
  url          = {https://proceedings.neurips.cc/paper/2002/hash/43e4e6a6f341e00671e123714de019a8-Abstract.html},
  bibsource    = {dblp computer science bibliography, https://dblp.org}
}

@book{ya1963statistical,
  title={Statistical Analysis: With Business and Economic Applications},
  author={Ya-Lun, Chou},
  year={1963},
  publisher={Holt, Rinehart and Winston}
}

@book{AP91,
  author = {{A. Papoulis}},
  description = {Livres : acoustiques, vibrations, etc. (PAG)},
  publisher = {McGraw-Hill},
  title = {Probability, random variables, and stochastic processes},
  year = 1991
}

@inproceedings{DBLP:conf/nips/GrillSATRBDPGAP20,
  author       = {Jean{-}Bastien Grill and
                  Florian Strub and
                  Florent Altch{\'{e}} and
                  Corentin Tallec and
                  Pierre H. Richemond and
                  Elena Buchatskaya and
                  Carl Doersch and
                  Bernardo {\'{A}}vila Pires and
                  Zhaohan Guo and
                  Mohammad Gheshlaghi Azar and
                  Bilal Piot and
                  Koray Kavukcuoglu and
                  R{\'{e}}mi Munos and
                  Michal Valko},
  editor       = {Hugo Larochelle and
                  Marc'Aurelio Ranzato and
                  Raia Hadsell and
                  Maria{-}Florina Balcan and
                  Hsuan{-}Tien Lin},
  title        = {Bootstrap Your Own Latent - {A} New Approach to Self-Supervised Learning},
  booktitle    = {Advances in Neural Information Processing Systems 33: Annual Conference
                  on Neural Information Processing Systems 2020, NeurIPS 2020, December
                  6-12, 2020, virtual},
  year         = {2020},
  url          = {https://proceedings.neurips.cc/paper/2020/hash/f3ada80d5c4ee70142b17b8192b2958e-Abstract.html},
  bibsource    = {dblp computer science bibliography, https://dblp.org}
}

@inproceedings{DBLP:conf/icml/SutskeverMDH13,
  author       = {Ilya Sutskever and
                  James Martens and
                  George E. Dahl and
                  Geoffrey E. Hinton},
  title        = {On the importance of initialization and momentum in deep learning},
  booktitle    = {Proceedings of the 30th International Conference on Machine Learning,
                  {ICML} 2013, Atlanta, GA, USA, 16-21 June 2013},
  series       = {{JMLR} Workshop and Conference Proceedings},
  volume       = {28},
  pages        = {1139--1147},
  publisher    = {JMLR.org},
  year         = {2013},
  url          = {http://proceedings.mlr.press/v28/sutskever13.html},
  bibsource    = {dblp computer science bibliography, https://dblp.org}
}

@techreport{yang2025fate,
  author       = {Joyce Yang and Patrick Grother and Mei Ngan and Kayee Hanaoka and Austin Hom},
  title        = {{Face Analysis Technology Evaluation (FATE) Part 11: Face Image Quality Vector Assessment – Specific Image Defect Detection}},
  institution  = {National Institute of Standards and Technology (NIST)},
  type         = {NIST Internal Report},
  number       = {NIST IR 8485 DRAFT SUPPLEMENT},
  address      = {Gaithersburg, MD},
  year         = {2025},
  month        = {April},
  url          = {https://doi.org/10.6028/NIST.IR.8485},
  note         = {Image Group, Information Access Division, Information Technology Laboratory. This publication is available free of charge.}
}

@article{DBLP:journals/tbbis/SchlettRTB24,
  author       = {Torsten Schlett and
                  Christian Rathgeb and
                  Juan E. Tapia and
                  Christoph Busch},
  title        = {Considerations on the Evaluation of Biometric Quality Assessment Algorithms},
  journal      = {{IEEE} Trans. Biom. Behav. Identity Sci.},
  volume       = {6},
  number       = {1},
  pages        = {54--67},
  year         = {2024},
  url          = {https://doi.org/10.1109/TBIOM.2023.3336513},
  doi          = {10.1109/TBIOM.2023.3336513},
  bibsource    = {dblp computer science bibliography, https://dblp.org}
}

@inproceedings{transface,
  author       = {Jun Dan and
                  Yang Liu and
                  Haoyu Xie and
                  Jiankang Deng and
                  Haoran Xie and
                  Xuansong Xie and
                  Baigui Sun},
  title        = {TransFace: Calibrating Transformer Training for Face Recognition from
                  a Data-Centric Perspective},
  booktitle    = {{IEEE/CVF} International Conference on Computer Vision, {ICCV} 2023,
                  Paris, France, October 1-6, 2023},
  pages        = {20585--20596},
  publisher    = {{IEEE}},
  year         = {2023},
  url          = {https://doi.org/10.1109/ICCV51070.2023.01887},
  doi          = {10.1109/ICCV51070.2023.01887},
  bibsource    = {dblp computer science bibliography, https://dblp.org}
}

@article{swinface,
  author       = {Lixiong Qin and
                  Mei Wang and
                  Chao Deng and
                  Ke Wang and
                  Xi Chen and
                  Jiani Hu and
                  Weihong Deng},
  title        = {SwinFace: {A} Multi-Task Transformer for Face Recognition, Expression
                  Recognition, Age Estimation and Attribute Estimation},
  journal      = {{IEEE} Trans. Circuits Syst. Video Technol.},
  volume       = {34},
  number       = {4},
  pages        = {2223--2234},
  year         = {2024},
  url          = {https://doi.org/10.1109/TCSVT.2023.3304724},
  doi          = {10.1109/TCSVT.2023.3304724},
  bibsource    = {dblp computer science bibliography, https://dblp.org}
}

@InProceedings{atzori2025vitfiqaassessingfaceimage,
    author    = {Atzori, Andrea and Boutros, Fadi and Damer, Naser},
    title     = {ViT-FIQA: Assessing Face Image Quality using Vision Transformers},
    booktitle = {Proceedings of the IEEE/CVF International Conference on Computer Vision (ICCV) Workshops},
    month     = {October},
    year      = {2025},
    pages     = {5935-5945}
}

@inproceedings{DBLP:conf/cvpr/KimS0JL24,
  author       = {Minchul Kim and
                  Yiyang Su and
                  Feng Liu and
                  Anil Jain and
                  Xiaoming Liu},
  title        = {KeyPoint Relative Position Encoding for Face Recognition},
  booktitle    = {{CVPR}},
  pages        = {244--255},
  publisher    = {{IEEE}},
  year         = {2024}
}

@article{DBLP:journals/ivc/ChettaouiDB25,
  author       = {Tahar Chettaoui and
                  Naser Damer and
                  Fadi Boutros},
  title        = {FRoundation: Are foundation models ready for face recognition?},
  journal      = {Image Vis. Comput.},
  volume       = {156},
  pages        = {105453},
  year         = {2025}
}

@InProceedings{mrfiqa,
    author    = {Ou, Fu-Zhao and Li, Chongyi and Wang, Shiqi and Kwong, Sam},
    title     = {MR-FIQA: Face Image Quality Assessment with Multi-Reference Representations from Synthetic Data Generation},
    booktitle = {Proceedings of the IEEE/CVF International Conference on Computer Vision (ICCV)},
    month     = {October},
    year      = {2025},
    pages     = {12915-12925}
}

\end{document}